\documentclass{article}

\usepackage{microtype}
\usepackage{graphicx}
\usepackage{subfigure}
\usepackage{booktabs} 

\usepackage{url}

\usepackage{mathtools}
\usepackage{amssymb}

\usepackage[utf8]{inputenc}

\usepackage{comment}

\usepackage{amsthm}

\usepackage{stmaryrd} 
 
\usepackage{xspace}

\newcommand{\p}[0]{\mathbb{P}}
\newcommand{\I}[0]{\mathbb{I}}
\newcommand{\E}[0]{\mathbb{E}}
\newcommand{\R}[0]{\mathbb{R}}
\newcommand{\N}[0]{\mathbb{N}}
\newcommand{\X}[0]{\mathcal{X}}

\newcommand{\Scal}[0]{\mathcal{S}}
\newcommand{\Rcal}[0]{\mathcal{R}}

\def\roc{{\rm ROC\xspace}}
\def\auc{{\rm AUC\xspace}}

\usepackage{multirow}
\usepackage{comment}

\newtheorem{theorem}{Theorem}
\newtheorem{definition}{Definition}
\newtheorem{lemma}{Lemma}
\newtheorem{remark}{Remark}
\newtheorem{proposition}{Proposition}

\DeclarePairedDelimiter\abs{\lvert}{\rvert}%
\DeclarePairedDelimiter\norm{\lVert}{\rVert}%

\makeatletter
\let\oldabs\abs
\def\abs{\@ifstar{\oldabs}{\oldabs*}}
\let\oldnorm\norm
\def\norm{\@ifstar{\oldnorm}{\oldnorm*}}
\makeatother

\DeclareMathOperator*{\argmax}{arg\,max}

\usepackage[capitalise]{cleveref}
\usepackage{color}

\newtheorem*{noise-assumption}{Noise Assumption (NA)}

\usepackage[accepted]{icml2018}

\icmltitlerunning{A Probabilistic Theory of Supervised Similarity Learning for Pointwise ROC Curve Optimization}

\begin{document}

\twocolumn[





\icmltitle{A Probabilistic Theory of Supervised Similarity Learning\\ for Pointwise ROC Curve Optimization}

\begin{icmlauthorlist}
\icmlauthor{Robin Vogel}{telecom,idemia}
\icmlauthor{Aur\'elien Bellet}{inria}
\icmlauthor{St\'ephan Cl\'emen\c con}{telecom}
\end{icmlauthorlist}

\icmlaffiliation{telecom}{T\'el\'ecom ParisTech, Paris, France}
\icmlaffiliation{inria}{INRIA, France}
\icmlaffiliation{idemia}{IDEMIA, Colombes, France}

\icmlcorrespondingauthor{Robin Vogel}{robin.vogel@telecom-paristech.fr}

\icmlkeywords{Machine Learning, ICML}

\vskip 0.3in
]



\printAffiliationsAndNotice{}  


\begin{abstract}
The performance of many machine learning techniques depends on the choice of an appropriate similarity or distance measure on the input space. Similarity learning (or metric learning) aims at building such a measure from training data so that observations with the same (resp. different) label are as close (resp. far) as possible.
In this paper, similarity learning is investigated from the perspective of pairwise bipartite ranking, where the goal is to rank the elements of a database by decreasing order of the probability that they share the same label with some query data point, based on the similarity scores. A natural performance criterion in this setting is pointwise $\roc$ optimization: maximize the true positive rate under a fixed false positive rate. We study this novel perspective on similarity learning through a rigorous probabilistic framework. The empirical version of the problem gives rise to a constrained optimization formulation involving $U$-statistics, for which we derive universal learning rates as well as faster rates under a noise assumption on the data distribution. We also address the large-scale setting by analyzing the effect of sampling-based approximations. Our theoretical results are supported by illustrative numerical experiments.
\end{abstract}


\section{Introduction} \label{sec:introduction}

Similarity (or distance) functions play a key role in many machine learning algorithms for problems ranging from classification (e.g., $k$-nearest neighbors) and clustering (e.g., $k$-means) to dimensionality reduction \cite{Maaten2008a} and ranking \cite{Chechik2010a}.
The success of such methods are heavily dependent on the relevance of the similarity function to the task and dataset of interest. This has motivated the research in similarity and distance metric learning \cite{Bellet2015c}, a line of work which consists in automatically learning a similarity function from data. This training data often comes in the form of pairwise similarity judgments derived from labels, such as positive (resp. negative) pairs composed of two instances with same (resp. different) label. Most existing algorithms can then be framed as unconstrained optimization problems where the objective is to minimize some average loss function over the set of similarity judgments \citep[see for instance][for methods tailored to classification]{Goldberger2004a,Weinberger2009a,Bellet2012a}. Some generalization bounds for this class of methods have been derived, accounting for the specific dependence structure found in the training similarity judgments \cite{Jin2009a,Bellet2015a,Cao2016a,Jain2017a,Verma2015}. We refer to \citet{Kulis2012a} and \citet{Bellet2015c} for detailed surveys on similarity and metric learning.

In this paper, we study similarity learning from the perspective of \emph{pairwise bipartite ranking}, where the goal is to rank the elements of a database by decreasing order of the probability that they share the same label with some query data point. This problem is motivated by many concrete applications: for instance, biometric identification aims to check the claimed identity of an individual by matching her biometric information (e.g., a photo taken at an airport) with a large reference database of authorized people (e.g., of passport photos) \cite{Jain2011a}. Given a similarity function and a threshold, the database elements are ranked in decreasing order of similarity score with the query, and the matching elements are those whose score is above the threshold.
In this context, performance criteria are related to the $\roc$ curve associated with the similarity function, i.e., the relation between the false positive rate and the true positive rate. Previous approaches have empirically tried to optimize the Area under the ROC curve ($\auc$) of the similarity function \cite{McFee2010a,Huo2018a}, without establishing any generalization guarantees. The $\auc$ is a global summary of the $\roc$ curve which penalizes pairwise ranking errors regardless of the positions in the list. More local versions of the $\auc$ (e.g., focusing on the top of the list) are difficult to optimize in practice and lead to complex nonconvex formulations \cite{CV07,Huo2018a}.
In contrast, the performance criterion we consider in this work is \emph{pointwise $\roc$ optimization}, which aims at maximizing the true positive rate under a fixed false positive rate. This objective, formulated as a constrained optimization problem, naturally expresses the operational constraints present in many practical scenarios. For instance, in biometric applications such as the one outlined above, the verification system is typically set to keep the proportion of people falsely considered a match below a predefined acceptable threshold \citep[see e.g.,][]{Jain2000a,Jain2014a}.


In addition to proposing an appropriate probabilistic framework to study this novel perspective on similarity learning, we make the following key contributions:

\textbf{Universal and fast learning rates.}
We derive statistical guarantees for the approach of solving the constrained optimization problem corresponding to the empirical version of our theoretical objective, based on a dataset of $n$ labeled data points. As the empirical quantities involved are not i.i.d. averages but rather in the form of $U$statistics \citep{Lee90}, our results rely on concentration bounds developed for $U$-processes \citep{Clemencon08Ranking}. We first derive a learning rate of order $O(1/\sqrt{n})$ which holds without any assumption on the data distribution. We then show that one can obtain faster rates under a low-noise assumption on the data distribution, which has the form of a margin criterion involving the conditional quantile. We are unaware of previous results of this kind for constrained similarity/distance metric learning.
Interestingly, we are able to illustrate the faster rates empirically through numerical simulations, which is rarely found in the literature on fast learning rates.

\textbf{Scalability by sampling.}
We address scalability issues that arise from the very large number of negative pairs when the dataset and the number of classes are large. In particular, we show that using an approximation of the pairwise negative risk consisting of $O(n)$ randomly sampled terms, known as an incomplete $U$-statistic \citep[see][]{Blom76,Lee90}, is sufficient to maintain the universal learning rate of $O(1/\sqrt{n})$. We analyze two different choices of sampling strategies and discuss properties of the data distribution which can make one more accurate than the other.
We further provide numerical experiments to illustrate the practical benefits of this strategy.

The rest of this paper is organized as follows.
Section~\ref{sec:background} introduces the proposed probabilistic framework for similarity learning and draws connections to existing approaches. In Section~\ref{sec:generalization}, we derive universal and fast learning rates for the minimizer of the empirical version of our problem. Section~\ref{sec:scalability} addresses scalability issues through random sampling, and Section~\ref{sec:experiments} presents some numerical experiments.
Detailed proofs can be found in the supplementary material.


\section{Background and Preliminaries}
\label{sec:background}

In this section, we introduce the main notations and concepts involved in the subsequent analysis. We formulate the supervised similarity learning problem from the perspective of pairwise bipartite ranking, and highlight connections with some popular metric and similarity learning algorithms of the literature.
Here and throughout, the indicator function of any event $\mathcal{E}$ is denoted by $\mathbb{I}\{\mathcal{E}\}$, the Dirac mass at any point $x$ by $\delta_x$, and the pseudo-inverse of any cdf $F(u)$ on $\mathbb{R}$ by $F^{-1}(t)=\inf\{v\in \mathbb{R}:\; F(v)\geq t  \}$.

\subsection{Probabilistic Framework for Similarity Learning}

We consider the (multi-class) classification setting. The random variable $Y$ denotes the output label with values in the discrete set $\{1,\; \ldots,\; K\}$ with $K\geq 1$, and $X$ is the input random variable, taking its values in a feature space $\X\subset \mathbb{R}^d$ with $d\geq 1$ and modeling some information hopefully useful to predict $Y$. We denote by $\mu(dx)$ the marginal distribution of $X$ and by $\eta(x)=(\eta_1(x),\; \ldots,\; \eta_K(x))$ the posterior probability, where $\eta_k(x)=\mathbb{P}\{Y=k\mid X=x\}$ for $x\in \X$ and $k\in\{1,\; \ldots,\; K\}$. The distribution of the random pair $(X,Y)$ is entirely characterized by $P=(\mu,\eta)$. The probability of occurrence of an observation with label $k\in\{1,\; \ldots,\; K\}$ is assumed to be strictly positive and denoted by $p_k=\mathbb{P}\{ Y=k \}$, and the conditional distribution of $X$ given $Y=k$ is denoted by $\mu_k(dx)$. Equipped with these notations, we have $\mu=\sum_{k=1}^K p_k\mu_k$.

\paragraph{Optimal similarity measures.}
The objective of \textit{similarity learning} can be informally formulated as follows: the goal is to learn, from a training sample $\mathcal{D}_n=\{(X_1,Y_1),\; \ldots,\; (X_{n},Y_n)\}$ composed of $n\geq 1$ independent copies of $(X,Y)$,  a (measurable) similarity measure $S:\X\times \X\rightarrow \mathbb{R}_+$ such that given two independent pairs $(X,Y)$ and $(X',Y')$ drawn from $P$, the larger the similarity $S(X,X')$ between two observations, the more likely they are to share the same label. The set of all similarity measures is denoted by $\mathcal{S}$. The class $\mathcal{S}^*$ of optimal similarity rules naturally corresponds to the set of strictly increasing transforms $T$ of the pairwise posterior probability $\eta(x,x')=\mathbb{P}\{Y=Y'\mid (X,X')=(x,x') \}$, where $(X',Y')$ denotes an independent copy of $(X,Y)$:
$$\{T\circ \eta \, | \, T:Im(\eta) \rightarrow \mathbb{R}_+ \text{ borelian, strictly increasing}\},$$
and where $Im (\eta)$ denotes the support of $\eta(X,X')$'s distribution. With the notations previously introduced, we have $\eta(x,x')=\sum_{k=1}^K\eta_k(x)\eta_k(x')$ for all $(x,x')\in\X^2$. A similarity rule $S^*\in \mathcal{S}^*$ defines the optimal preorder\footnote{A preorder on a set $\X$ is any reflexive and transitive binary relationship on $\X$. A preorder is an order if, in addition, it is antisymmetrical.} $\preceq^*$ on the product space $\X\times \X$: for any $(x_1,x_2,x_3,x_4)\in\X^4$, $x_1$ and $x_2$ are more similar to each other than $x_3$ and $x_4$ iff $\eta(x_1, x_2)\geq \eta(x_3,x_4)$, and one writes $(x_3,x_4)\preceq^*(x_1,x_2)$ in this case. For any $x\in \mathcal{X}$, $S^*$ also defines a preorder $\preceq^*_{x}$ on the input space $\X$, permitting to rank optimally all possible observations by increasing degree of similarity to $x$: for all $(x_1,x_2)\in \X^2$, $x_1$ is more similar to $x$ than $x_2$ (one writes $x_2 \preceq^*_{x} x_1$) iff $(x,x_2)\preceq^*(x,x_1)$, meaning that $\eta(x,x_2)\leq \eta(x,x_1)$. We point out that, despite its simplicity, this framework covers a wide variety of applications, such as the biometric identification problem mentioned earlier in the introduction.

\paragraph{Similarity learning as pairwise bipartite ranking.}
In view of the objective formulated above, similarity learning can be seen as a \textit{bipartite ranking} problem on the product space $\X\times \X$ where, given two independent realizations $(X,Y)$ and $(X',Y')$ of $P$, the input r.v. is the pair $(X,X')$ and the binary label is $Z=2\mathbb{I}\{Y=Y'  \}-1$. One may refer to \textit{e.g.} \citet{CV09ieee} and the references therein for a statistical learning view of bipartite ranking. $\roc$ analysis is the gold standard to evaluate the performance of a similarity measure $S$ in this context, \textit{i.e.} to measure how close the preorder induced by $S$ is to $\preceq^*$. The $\roc$ curve of $S$ is the PP-plot $t\in \mathbb{R}_+\mapsto (F_{S,-}(t),F_{S,+}(t))$,
where, for all $t\geq 0$,
\begin{eqnarray*}
F_{S,-}(t)&=&\mathbb{P}\{ S(X,X')> t  \mid Z=-1 \},\\
F_{S,+}(t)&=& \mathbb{P}\{ S(X,X')> t  \mid Z=+1 \},
\end{eqnarray*}
where possible jumps are connected by line segments. Hence, it can be viewed as the graph of a continuous function $\alpha\in (0,1)\mapsto \roc_S(\alpha)$, where $\roc_S(\alpha)=F_{S,+}\circ F_{S,-}^{-1}(\alpha)$ at any point $\alpha\in (0,1)$ such that $F_{S,-}\circ F_{S,-}^{-1}(\alpha)=\alpha$. The curve $\roc_S$ reflects the ability of $S$ to discriminate between pairs with same labels and pairs with different labels: the stochastically smaller  than $F_{S,-}$ the distribution $F_{S,+}$ is, the higher the associated $\roc$ curve. Note that it corresponds to the type I error vs power plot of the statistical test $\mathbb{I}\{ S(X,X')> t \}$ when the null hypothesis stipulates that $X$ and $X'$ have different marginal distribution (\textit{i.e.}, $Y\ne Y'$). A similarity measure $S_1$ is said to be more accurate than another similarity $S_2$ when $\roc_{S_2}(\alpha)\leq \roc_{S_1}(\alpha)$ for any $\alpha\in (0,1)$. A straightforward Neyman-Pearson argument shows that $\mathcal{S}^*$ is the set of optimal elements regarding this partial order on $\mathcal{S}$: $\forall (S,S^*)\in \mathcal{S}\times \mathcal{S}^*$, $\roc_{S}(\alpha)\leq \roc_{S^*}(\alpha)=\roc_{\eta}(\alpha)$ for all $\alpha\in (0,1)$. For simplicity, we will assume that the conditional cdf of $\eta(X,X')$ given $Z=-1$ is invertible.

\paragraph{Pointwise ROC optimization.}
In many applications, one is interested in finding a similarity function which optimizes the $\roc$ curve at a particular point $\alpha\in (0,1)$.
The superlevel sets of similarity functions in $\mathcal{S}^*$ define the solutions of pointwise $\roc$ optimization problems in this context.
In the above framework, it indeed follows from Neyman Pearson's lemma that the test statistic of type I error less than $\alpha$ with maximum power is the indicator function of the set $\mathcal{R}^*_{\alpha}=\{(x,x')\in \X^2:\; \eta(x,x')\geq Q^*_{\alpha}  \}$, where $Q^*_{\alpha} $ is the conditional quantile of the r.v. $\eta(X,X')$ given $Z=-1$ at level $1-\alpha$. Restricting our attention to similarity functions bounded by $1$, this corresponds to the unique solution of the following problem:
\begin{equation} \label{eq:point_opt}
\max_{S:\X^2\rightarrow [0,1],\text{ borelian}} R^+(S) \quad \text{subject to} \quad R^-(S) \leq \alpha,
\end{equation}
where $R^+(S)=\mathbb{E}[S(X,X') \mid Z=+1]$ is referred to as \emph{positive risk} and $R^-(S)=\mathbb{E}[S(X,X') \mid Z=-1 ]$ as the \emph{negative risk}.

 \begin{remark}\label{rk:cost_sens} {\sc (Unconstrained formulation)} The superlevel set $ \mathcal{R}^*_{\alpha}$ of the pairwise posterior probability $\eta(x,x')$ is the measurable subset $\mathcal{R}$ of $\X^2$ that minimizes the cost-sensitive classification risk:
 \begin{multline*}
 p(1-Q^*_{\alpha} )\mathbb{P}\left\{ (X,X')\notin \mathcal{R} \mid Z=+1  \right\} + \\ 
 (1-p)Q^*_{\alpha} \mathbb{P}\left\{ (X,X')\in \mathcal{R} \mid Z=-1  \right\},
 \end{multline*}
 where $p=\mathbb{P}\{Z=+1 \}=\sum_{k=1}^Kp_k^2$. Notice however that the asymmetry factor, namely the quantile $Q^*_{\alpha}$, is unknown in practice, just like the r.v. $\eta(X,X')$. For this reason, one typically considers the problem of maximizing
 \begin{equation}\label{eq:comb}
 R^+(S) -\lambda R^-(S),   
 \end{equation}
 for different values of the constant $\lambda>0$. The performance in terms of $\roc$ curve can only be analyzed \emph{a posteriori}, and the value $\lambda$ thus needs to be tuned empirically by model selection techniques.
 \end{remark}

\subsection{Connections to Existing Similarity and Metric Learning Approaches}

We point out that the similarity learning framework described above can be equivalently described in terms of learning a dissimilarity measure (or pseudo distance metric) $D:\mathcal{X}\times\mathcal{X}\rightarrow\R_+$. In this case, the pointwise $\roc$ optimization problem \eqref{eq:point_opt} translates into:
\begin{multline} \label{eq:point_opt2}
\min_{D:\X^2\rightarrow [0,1]}\mathbb{E}\left[D(X,X')\mid Z=+1  \right]\\
 \text{subject to } \mathbb{E}\left[ D(X,X')\mid Z=-1  \right] \geq 1 - \alpha.
\end{multline}

A large variety of practical similarity and distance metric learning algorithms have been proposed in the literature, all revolving around the same idea that a good similarity function should output large scores for pairs of points in the same class, and small scores for pairs with different label. They differ from one another by the class of metric/similarity functions considered, and by the kind of objective function they optimize \citep[see][for a comprehensive review]{Bellet2015c}.
In any case, $\roc$ curves are commonly used to evaluate metric learning algorithms when the number of classes is large \citep[see for instance][]{IsThatYou,Kostinger2012a,Shen2012a}, which makes our framework very relevant in practice.
Several popular algorithms optimize an empirical version of Problems \eqref{eq:point_opt}-\eqref{eq:point_opt2}, often in their unconstrained version as in \eqref{eq:comb} \citep{Liu2010a,Xie2015a}. We argue here in favor of the constrained version as the parameter $\alpha$ has a direct correspondence with the point $\roc_S(\alpha)$ of the $\roc$ curve, unlike the unconstrained case (see Remark~\ref{rk:cost_sens}). This will be illustrated in our numerical experiments of Section~\ref{sec:experiments}.

Interestingly, our framework sheds light on MMC, the seminal metric learning algorithm of \citet{Xing2002a} originally designed for clustering with side information. MMC solves the empirical version of \eqref{eq:point_opt2} with $\alpha$ fixed to $0$. This is because MMC optimizes over a class of distance functions with unbounded values, hence modifying $\alpha$ does not change the solution (up to a scaling factor). We note that by choosing a bounded family of distance functions, one can use the same formulation to optimize the pointwise $\roc$ curve.

\section{Statistical Guarantees for Generalization}
\label{sec:generalization}

Pointwise $\roc$ optimization problems have been investigated from a statistical learning perspective by \citet{ScottNowak} and \citet{Clemencon2010} in the context of binary classification. The major difference with the present framework lies in the pairwise nature of the quantities appearing in Problem \eqref{eq:point_opt} and, consequently, in the complexity of its empirical version. In particular, natural statistical estimates for the positive risk $R^+(S)$ and the negative risk $R^-(S)$ \eqref{eq:point_opt} computed on the training sample $\mathcal{D}_n=\{(X_1,Y_1),\; \ldots,\; (X_{n},Y_n)\}$ are given by:
\begin{eqnarray}
\widehat{R}^+_n(S)&=&\frac{1}{n_+}\sum_{1\leq i<j\leq n}S(X_i,X_j)\cdot \mathbb{I}\{ Y_i=Y_j \},\label{eq:posemprisk}\\
\widehat{R}^-_n(S)&=& \frac{1}{n_-}\sum_{1\leq i<j\leq n}S(X_i,X_j)\cdot \mathbb{I}\{ Y_i\neq Y_j \},\label{eq:negemprisk}
\end{eqnarray}
where $n_+=\sum_{1\leq i<j\leq n}\mathbb{I}\{ Y_i=Y_j  \}=n(n-1)/2-n_-$. 
It is important to note that these quantities are not i.i.d. averages, since several pairs involve each i.i.d. sample. This breaks the analysis carried out by \citet[][Section~5 therein]{Clemencon2010} for the case of binary classification.

We can however observe that $U^+_n(S)=2n_+/(n(n-1))\widehat{R}^+_n(S)$ and $U^-_n(S)=2n_-/(n(n-1))\widehat{R}^-_n(S)$ are $U$-statistics of degree two with respective symmetric kernels $h_+((x,y),(x',y'))=S(x,x')\cdot \mathbb{I}\{ y=y' \}$ and $h_-((x,y),(x',y'))=S(x,x')\cdot \mathbb{I}\{ y\neq y' \}$.\footnote{We give the definition of $U$-statistics in the supplementary material for completeness.}
 We will therefore be able to use existing representation tricks to derive concentration bounds for $U$-processes (collections of $U$-statistics indexed by classes of kernel functions), under appropriate complexity conditions, see \textit{e.g.} \cite{Dud99}.

We thus investigate the generalization ability of solutions obtained by solving the empirical version of Problem \eqref{eq:point_opt}, where we also restrict the domain to a subset $\mathcal{S}_0\subset \mathcal{S}$ of similarity functions bounded by $1$, and we assume $\mathcal{S}_0$ has controlled complexity (\textit{e.g.} finite {\sc VC} dimension). 
Finally, we replace the target level $\alpha$ by $\alpha+\Phi$, where $\Phi$ is some tolerance parameter that should be of the same order as the maximal deviation $\sup_{S\in \mathcal{S}_0}\vert \widehat{R}^-_n(S)-R^-(S) \vert$. This leads to the following empirical problem:
\begin{equation}
\max_{S \in \mathcal{S}_0} \widehat{R}^+_n(S) \quad \text{subject to} \quad \widehat{R}^-_n(S)\le \alpha + \Phi .
\label{simlearn}
\end{equation}

Following \citet{Clemencon08Ranking}, we have the following lemma.
 
\begin{lemma}\label{lem:Ubounds} (\citealp{Clemencon08Ranking}, Corollary~3) Assume that $\mathcal{S}_0$ is a VC-major class of functions with finite {\sc VC} dimension $V<+\infty$. We have with probability larger than $1-\delta$: $\forall n>1$,
\begin{equation}\label{eq:Ubounds}
\sup_ {S\in \mathcal{S}_0}\left\vert\widehat{U}^+_n(S) -  \mathbb{E}[\widehat{U}^+_n(S)] \right\vert \leq 2C\sqrt{\frac{V}{n}}+2 \sqrt{\frac{\log (1/\delta)}{n-1}},
\end{equation}
where $C$ is a universal constant, explicited in \citet[][page 198 therein]{Bousquet2004}.
\end{lemma}
A similar result holds for the $U$-process $\{ \widehat{U}^-_n(S) -  U^- (S)\}_{S\in \mathcal{S}_0}$.
We are now ready to state our universal learning rate, describing the generalization capacity of solutions of the constrained optimization program \eqref{simlearn} under specific conditions for the class $\mathcal{S}_0$ of similarity functions and a suitable choice of the tolerance parameter $\Phi$. This result can be established by combining Lemma~\ref{lem:Ubounds} with the derivations of \citet[][Theorem~10 therein]{Clemencon2010}. Details can be found in the supplementary material.

\begin{theorem}  \label{ccl-slow-rates}
Suppose that the assumptions of Lemma~\ref{lem:Ubounds} are fulfilled and that $S(x,x')\leq 1$ for all $S\in \mathcal{S}_0$ and any $(x,x')\in \mathcal{X}^2$. Assume also that there exists a constant $\kappa\in (0,1)$ such that $\kappa \leq \sum_{k=1}p_k^2\leq 1-\kappa$.
For all $\delta \in (0,1)$ and $n>1$, set
  \begin{equation*}
  \Phi_{n,\delta} =  2 C \kappa^{-1} \sqrt{\frac{V}{n}} + 2\kappa^{-1}(1+\kappa^{-1}) 
  \sqrt{\frac{\log(3/\delta)}{n-1}}, 
  \end{equation*}
  and consider a solution $\hat{S}_n$ of the contrained minimization problem $\eqref{simlearn}$ with $\Phi=\Phi_{n,\delta/2}$.
Then, for any $\delta\in (0,1)$, we have simultaneously with probability at least $1-\delta$: $\forall n\geq 1
+ 4 \kappa^{-2} \log(3/\delta)$,
  \begin{multline}\label{eq1}
  R^+ (\hat{S}_n) \ge \roc_{S^*}(\alpha)  - \Phi_{n,\delta/2} \\-\Big\{ \roc_{S^*}(\alpha)-\sup_{S\in \mathcal{S}_0:\; R^-(S)\leq \alpha} R^+(S) \Big\}, 
  \end{multline}
  and 
  \begin{equation}\label{eq2}
  R^-(\hat{S}_n) \le \alpha + \Phi_{n,\delta/2}. 
  \end{equation}
\end{theorem}

\begin{remark}{\sc (On bias and model selection)} We point out that the last term on the right hand side of \eqref{eq1} should be interpreted as the bias of the statistical learning problem \eqref{simlearn}, which depends on the richness of class $\mathcal{S}_0$. This term vanishes when $\mathbb{I}\{ (x,x')\in \mathcal{R}^*_{\alpha} \}$ belongs to $\mathcal{S}_0$. Choosing a class yielding a similarity rule of highest true positive rate with large probability can be tackled by means of classical model selection techniques, based on resampling methods or complexity penalization (note that oracle inequalities can be straightforwardly derived from the same analysis). 
\end{remark}

Except for the minor condition stipulating that the probability of occurrence of ``positive pairs'' $\sum_{k=1}^Kp_k^2$ stays bounded away from $0$ and $1$, the generalization bound stated in Theorem~\ref{ccl-slow-rates} holds whatever the probability distribution of $(X,Y)$. 
Beyond such universal results, we investigate situations where rates faster than $O(1/\sqrt{n})$ can be achieved by solutions of \eqref{simlearn}. Such fast rates results exist for binary classification under the so-called Mammen-Tsybakov noise condition, see \textit{e.g.} \citet{Bousquet2004} for details. By means of a variant of the Bernstein inequality for $U$-statistics, we can establish fast rate bounds under the following condition on the data distribution.
\begin{noise-assumption}
	There exist a constant $c$ and $a \in [0,1]$ such that, almost surely, 
\begin{align*}
\E_{X'} \big [ \abs{\eta(X,X')-Q^*_\alpha}^{-a} \big ] \le c. 
\end{align*}\label{noise-ass-weaker}
\end{noise-assumption}

This noise condition is similar to that introduced by \citet{Mammen1} for the binary classification framework, except that the threshold $1/2$ is replaced here by the conditional quantile $Q^*_{\alpha}$. It characterizes ``nice'' distributions for the problem of $\roc$ optimization at point $\alpha$: it essentially ensures that the pairwise posterior probability is bounded away from $Q^*_{\alpha}$ with high probability. Under the assumption, we can derive the following fast learning rates.
\begin{theorem}
Suppose that the assumptions of Theorem \ref{ccl-slow-rates} are satisfied, that condition NA holds true and that the optimal similarity rule $S^*_{\alpha}(x,x')=\mathbb{I}\{(x,x')\in \mathcal{R}^*_{\alpha}  \}$ belongs to $\mathcal{S}_0$. Fix $\delta > 0 $. Then, there exists a constant $C'$, depending on $\delta$, $\kappa$, $Q^*_{\alpha}$, $a$, $c$ and $V$ such that, with
probability at least $1-\delta$,
\begin{multline*}
\roc_{S^*}(\alpha) -R^+(\hat{S}_n)\le  C'n^{-(2+a)/4}, \\
\text{and} \quad R^-(\hat{S}_n) \le \alpha + 2 \Phi_{n,\delta/2}.
\end{multline*}
\label{ccl-fast-rates}
\end{theorem}
\begin{remark}{\sc (On the NA condition)}
The noise condition is automatically fulfilled for any $a \in (0,1)$ when, for almost every point $x$ with respect to the 
measure induced by $X$, $\eta(x,X')$ has an absolutely continuous distribution and bounded density. 
This assumption means that the problem of ranking by similarity to an instance $x$ is not too hard 
for any value of $x$, see supplementary material for more details.
\end{remark}
The proof is based on the same argument as that of \citet[][Theorem 12 therein]{Clemencon2010}, except that it involves a sharp control of the fluctuations of the $U$-statistic estimates of the true positive rate excess $\roc_{S^*}(\alpha) -R^+(S)$ over the class $\mathcal{S}_0$. The reduced variance property of $U$-statistics plays a crucial role in the analysis, which essentially relies on the Hoeffding decomposition \citep[see][]{Hoeffding48}. Technical details can be found in the supplementary material.



\section{Scalability by Sampling Approximations}
\label{sec:scalability}

In the previous section, we analyzed the learning rates achieved by a minimizer of the empirical problem \eqref{simlearn}. In the large-scale setting, solving this problem can be computationally costly due to the very large number of training pairs. In particular, the positive and negative empirical risks $\widehat{R}_n^+(S)$ and $\widehat{R}_n^-(S)$ are sums over respectively $\sum_{k=1}^K n_k(n_k-1)/2$ and $\sum_{k<l} n_k n_l$ pairs.
We focus here more specifically on the setting where we have a large number of (rather balanced) classes, as in our biometric identification motivating example where a class corresponds to an identity. In this regime, we are facing a highly imbalanced problem since the number of negative pairs becomes overwhelmingly large compared to the number of positive pairs. For instance, even for the MNIST dataset where the number of classes is only $K=10$ and $n_k=6000$, there are already 10 times more negative pairs than positive pairs.

A natural strategy, often used by metric learning practitioners \citep[see e.g.,][]{Babenko2009a,Wu2013a,Xie2015a}, is to drastically subsample the negative pairs while keeping all positive pairs.
In this section, we shed light on this popular practice by analyzing the effect of subsampling (conditionally upon the data) the negative pairs onto the generalization performance.

A simple approach consists in replacing the empirical negative risk $\widehat{R}_n^-(S)$ by the following approximation:
\begin{align*}
\bar{R}_{B}^-(S) &:= \frac{1}{B} \sum_{(i, j) \in \mathcal{P}_B} S(X_{i}, X_j),
\end{align*}
where $\mathcal{P}_B$ is a set of cardinality $B$ built by sampling with replacement in the set of negative training pairs $\Lambda_{\mathcal{P}} =  \left \{ (i,j) \, | \, i,j\in\{1,\dots,n\}; Y_i \ne Y_j \right \}$. 
Conditioned upon the $n_k$'s,
$\bar{R}_{B}^-(S)$ can be viewed as an \emph{incomplete} version of the $U$-statistic $\widehat{R}_n^-(S)$ consisting of $B$ pairs \citep{Blom76,Lee90}. 

Despite the simplicity of the above approximation, we also consider an alternative sampling strategy, which consists in sampling a number $B$ of $K$-tuples containing one random sample of each class. Formally, this corresponds to the following approximation:
\begin{align*}
\widetilde{R}_{B}^-(S) &:= \frac{1}{B} \sum_{(i_1, \dots i_K) \in \mathcal{T}_B} h_S(X_{i_1}, \dots , X_{i_K}),
\end{align*}
where $h_S(X_{1}, \dots , X_{K}) = \frac{1}{n_-} \sum_{k<l} n_k n_l S(X_k,X_l)$ and
 $\mathcal{T}_B$ is a set of cardinality $B$ built by sampling with replacement in the set of $K$-tuples
$\Lambda_{\mathcal{T}} = \left \{ (i_1, \dots, i_K ) \, | \, i_k \in\{ 1, \dots, n_k \} ; k = 1, \dots, K \right \}$. $\widetilde{R}_{B}^-(S)$ is also an incomplete version of $\widehat{R}_n^-(S)$, with the alternative view of $\widehat{R}_n^-(S)$ as a \emph{generalized} $K$-sample $U$-statistic  \citep{Lee90} of degree $(1,\dots,1)$ and kernel $h_S$, see supplementary material for a full definition. Note that $\widetilde{R}_{B}^-(S)$ contains $BK(K-1)/2$ pairs, balanced across all class pairs.

$\bar{R}_{B}^-(S)$ and $\widetilde{R}_{B}^-(S)$ are both unbiased estimates of $\widehat{R}_n^-(S)$, but their variances are different and one approximation might be better than the other in some regimes.
The following result provides expressions for the variances of both incomplete estimators for a fixed budget of $B_0$ sampled pairs, 
under a standard asymptotic framework.

\begin{proposition}
\label{prop:var}
Let $B_0$ be the number of pairs sampled in both schemes, and denote $V_n = \text{Var} ( \hat{R}_n^-(S)  )$. 
When $B_0/n \to 0$, $n\to \infty$ and for all $k\in\{1,\dots, K \}, \; n_k/n \to p_k > 0$, we have:
\begin{align*}
 \text{Var} ( \widetilde{R}_B^-(S)  )  -  V_n   & \sim  
 \frac{ K(K-1)}{2B_0}\text{Var}  ( h_S  ( X^{(1)}, \dots , X^{(k)}  ) ), \\
  \text{Var} (\bar{R}_B^-(S) )  - V_n & \sim  
  B_0^{-1}  \text{Var}  ( S(X,X') \, | \, Y \ne Y'  ),   \\
\end{align*}
where $X^{(k)}$ denotes $X \, | \, Y=k$ for all $k \in \{ 1 , \dots , K \}$.
\end{proposition}

Proposition~\ref{prop:var} states that if the variance of similarity scores on 
the negative pairs is high compared to the variance of a weighted average of similarity scores on all types 
$(k,l)$ of negative pairs, then one should prefer tuple-based sampling (otherwise pair-based sampling is better).
As an example, consider the case where the similarity scores on the negative pairs constructed from classes $(k_0,l_0)$
are consistently higher than for other negative pairs. These high similarity pairs will not be sampled very often by the pair-based sampling method, in contrast to the tuple-based approach. In that scenario, the variance of $S(X,X') \,| \, Y\ne Y'$ 
is high while the variance of $ h_S \left ( X^{(1)}, \dots , X^{(k)} \right )$ is low, and the tuple-based method 
should be preferred. In practice, the properties of the data should guide the choice of the sampling approach.


We now analyze the effect of sampling on the performance of the empirical risk minimizer. We consider tuple-based sampling (results of the same order can be obtained for pair-based sampling). Let $\tilde{S}_B$ be the minimizer of the following simpler empirical problem:
\begin{align}
\argmax_{S\in\mathcal{S}_0} \widehat{R}_n^+(S)\quad \text{subject to}  
\; \widetilde{R}_B^-(S) \le \alpha + \Phi_{n, \delta, B}.
\label{emp-incomplete-simlearn}
\end{align}

We have the following theorem, based on combining Theorem~\ref{ccl-slow-rates} with a result bounding the maximal deviation between $\widehat{R}_n^-(S)$ and its incomplete version $\widetilde{R}_B^-(S)$, see \citet{JMLRincompleteUstats}.

\begin{theorem}
\label{thm:incomplete}
Let $N = \min_{1 \le k \le K} n_k$ and $\alpha \in (0,1)$, assume that $S^* \in \mathcal{S}_0$ and that $\mathcal{S}_0$ is a VC-major class of dimension $V$. For all $(\delta,n, B) \in (0,1) \times \N^* \times \N^*$, set
\begin{align*}
\Phi_{n,\delta, B} = 4 \sqrt{ \frac{V \log(1+N)}{N}} + \sqrt{\frac{\log(2/\delta)}{N}} \\+ 
\sqrt{ 2\frac{V \log (1 + \prod_{k=1}^K n_k) + \log (4/\delta)}{B} }.
\end{align*}
Then we have simultaneously with probability at least $1-\delta$,
\begin{align*}
R^+(\tilde{S}_B) \ge R_*^+ - 2 \Phi_{n,\delta, B} \quad \text{and} \quad 
R^-(\tilde{S}_B) \le \alpha + 2 \Phi_{n, \delta, B}.
\end{align*}
\end{theorem}

This result is very similar to Theorem~\ref{ccl-slow-rates}, with an additive error term in $O(\sqrt{\log n/B})$. Remarkably, this implies that it is sufficient to sample $B=O(n)$ tuples (hence only $O(nK^2)$ pairs) to preserve the $O(\sqrt{\log n/n})$ learning rate achieved when using all negative pairs.
This will be confirmed empirically in our numerical experiments.


\begin{remark}[Approximating the positive risk]
When needed, sampling-based techniques can also be used to approximate the empirical positive risk $\widehat{R}_n^+(S)$, with generalization results analogous to Theorem~\ref{thm:incomplete}. Details are left to the reader.
\end{remark}


\section{Illustrative Experiments}
\label{sec:experiments}

In this section, we present some experiments to illustrate our main results. We first illustrate how solving instances of Problem \eqref{simlearn} allows to optimize for specific points of the ROC curve.
We then provide some numerical evidence of the fast rates of Theorem~\ref{ccl-fast-rates}.
Finally, we illustrate our scalability results of Section~\ref{sec:scalability} by showing that dramatically subsampling the negative empirical risk leads to negligible loss in generalization performance.

\subsection{Pointwise ROC Optimization}

\begin{figure}[t]
\centering
\subfigure[Simulated data (stars are class centroids)]{\label{fig:simulated_data}
\includegraphics[width=0.43\columnwidth]{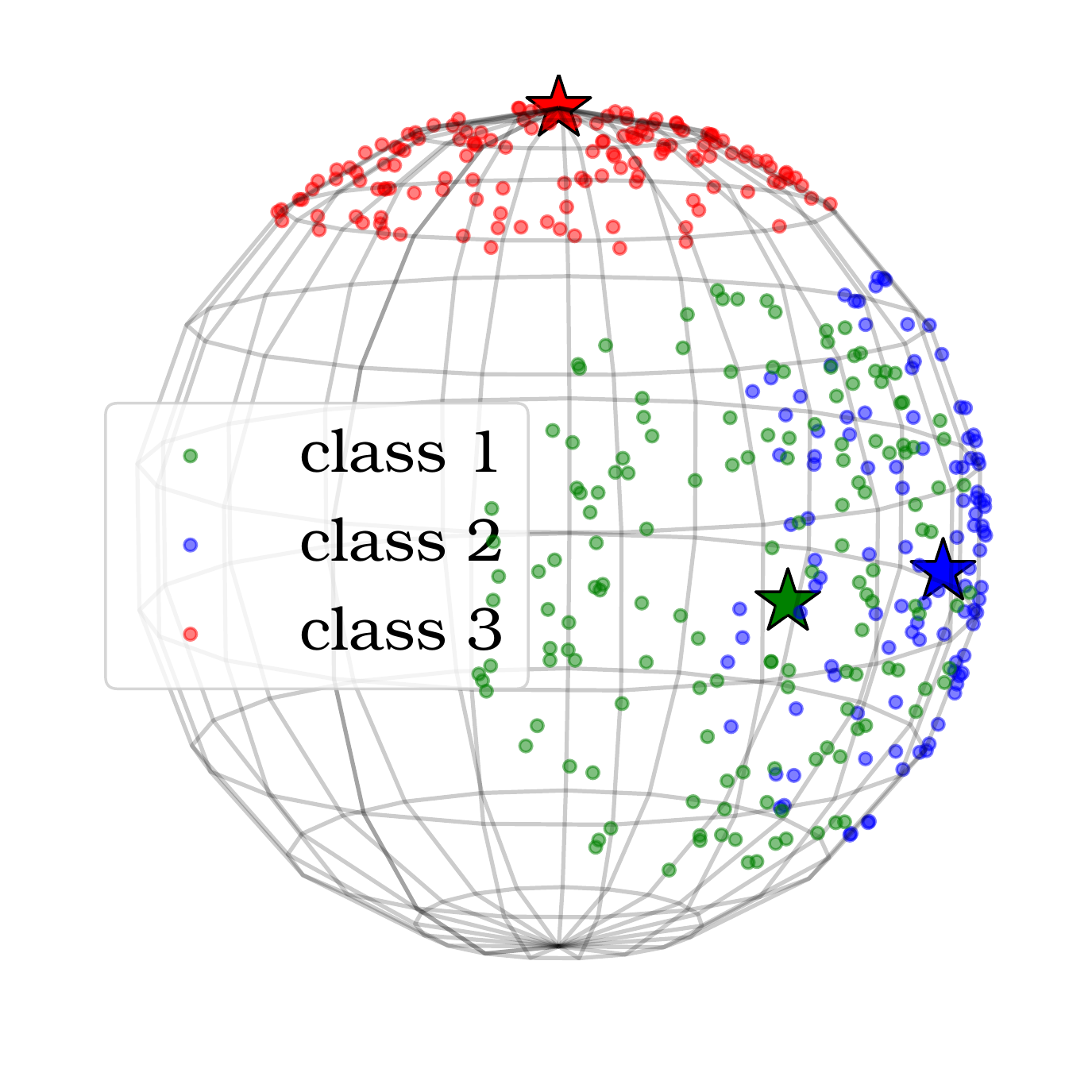}}\hspace{-.5cm}
\subfigure[ROC curves]{\label{fig:roc_curves_simdata}
\includegraphics[width=0.58\columnwidth]{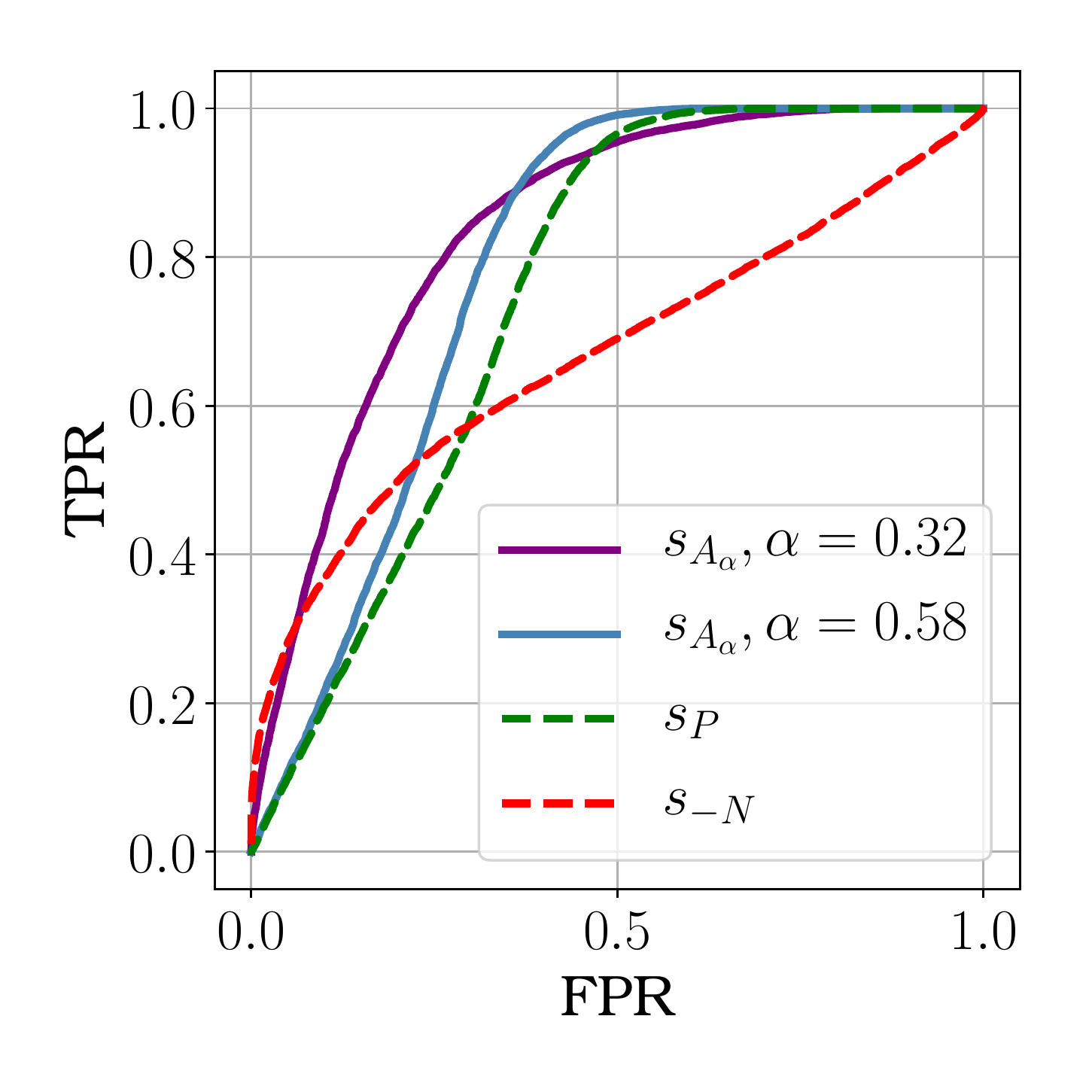}
}
\caption{Illustrative experiments for pointwise ROC optimization.}
\label{fig:simdata_exp}
\end{figure}

We illustrate on synthetic data that solving \eqref{simlearn} for different values of $\alpha$ can optimize for different regions of the ROC curve. Let $\mathcal{X}\subset \mathbb{R}^d$, and let $\mathcal{S}_0$ be the set of bilinear similarities
with norm-constrained matrices
\begin{align*}
\mathcal{S}_0 = \left \{ S_A : (x, x') \mapsto \frac{1}{2} \left ( 1  + x^\top A x' \right ) \; \big | \; 
\norm{A}_F^2 \le 1 \right \}, 
\end{align*}
where $\norm{A}_F^2 = \sum_{i,j=1}^d a_{ij}^2$. Note that when data is scaled ($\|x\|=1$ for all $x\in\X$), we have $S_A(x,x')\in[0, 1]$ for all $x,x'\in\X$ and all $S_A\in\mathcal{S}_0$.
In our simple experiment, we have $K=3$ classes and observations belong to the sphere in $\R^3$. Denoting by $\theta_{x,c_i}$ the angle between the element 
$x$ and the centroid $c_i$ of class $i$, we set for all $i \in \{ 1, 2, 3\}$, 
\begin{align*}
\mu_i (x) &\propto \I \left \{ \theta_{x,c_i} < \frac{\pi}{4}  \right \} , \quad  p_i = \frac{1}{3}
\end{align*}
 and $c_1 = \left( \cos(\pi/3), \sin(\pi/3) , 0 \right )$, $c_2 = e_2$, $c_3 = e_3$ with $e_i$ vectors of the standard basis of $\R^3$. 
 See Figure~\ref{fig:simulated_data} for a graphical representation of the data.

The solutions of the problem can be expressed in closed form using Lagrangian duality.
In particular, when the constraints are saturated, the solution $S_{A_\alpha}$ is an increasing transformation of $s_{P - \lambda_\alpha N}$ with 
\begin{align*}
P &= \frac{1}{2n_+} \sum_{1 \le i < j \le n} \I\left\{ Y_i = Y_j \right\} \cdot 
\left ( X_i X_j^\top  +  X_j  X_i^\top \right ), \\
N &= \frac{1}{2n_-} \sum_{1 \le i < j \le n} \I\left\{ Y_i \ne Y_j \right\} \cdot 
\left ( X_i X_j^\top  +  X_j  X_i^\top \right ),
\end{align*}
and $\lambda_{\alpha}$ is a positive Lagrange multiplier decreasing in $\alpha$, see supplementary material for details.
By varying $\alpha$, we trade-off between the information contained in the positive pairs 
($\alpha$ large, $\lambda_\alpha$ close to zero) and in the negative pairs 
($\alpha$ small, $\lambda_\alpha$ large), which indeed results in optimizing different areas of the ROC curve, see Figure~\ref{fig:roc_curves_simdata}.


\subsection{Fast Rates}
\label{experiments-fast-rates}

\begin{figure}[t]
\centering
\includegraphics[width=\columnwidth]{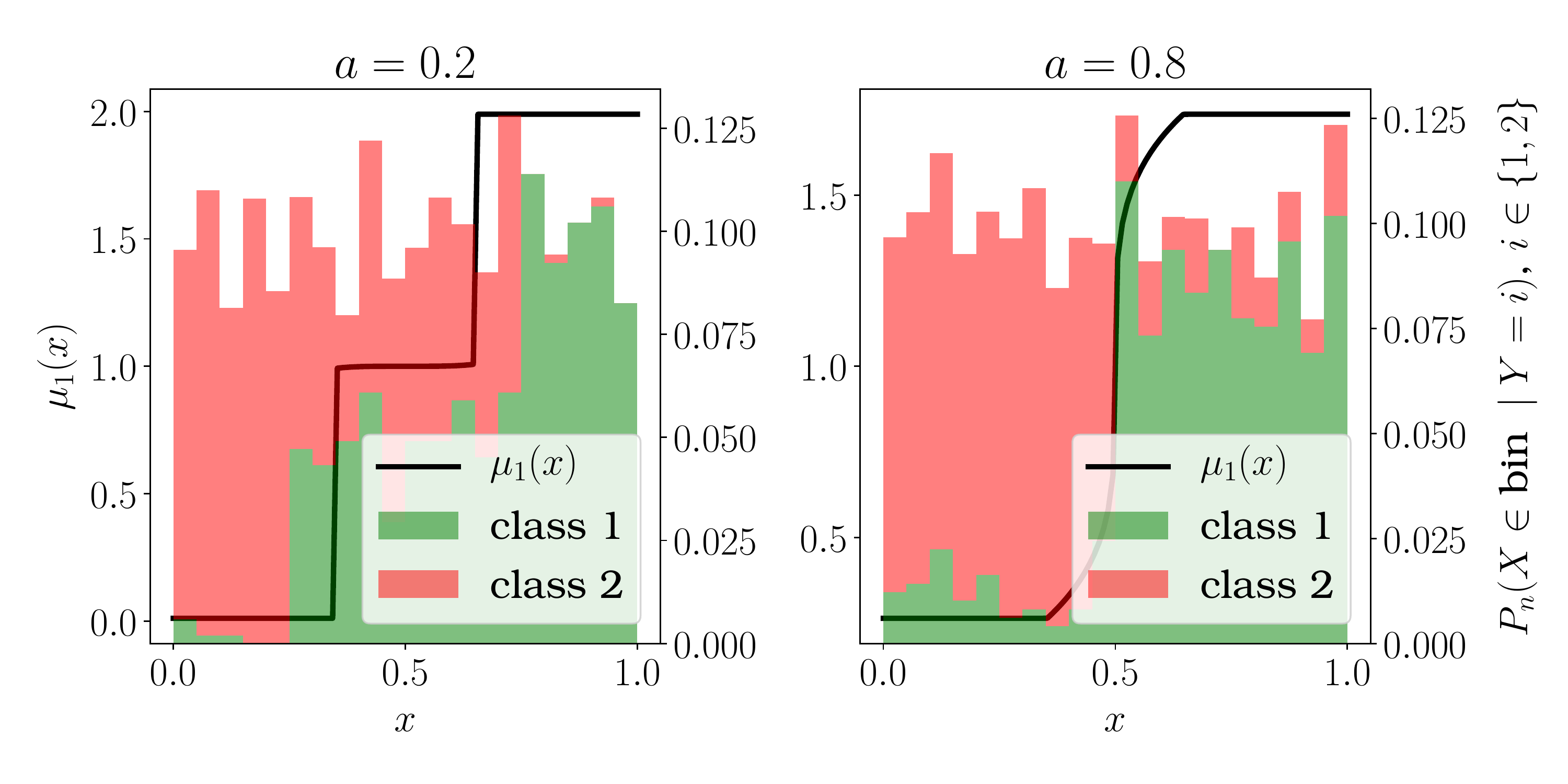}
\caption{Example distributions and $\mu_1$'s for $n=1000$ and two values of $a$.}
\label{fig:fast_example_distrib}
\end{figure}

\begin{figure}[t]
\centering
\includegraphics[width=0.75\columnwidth]{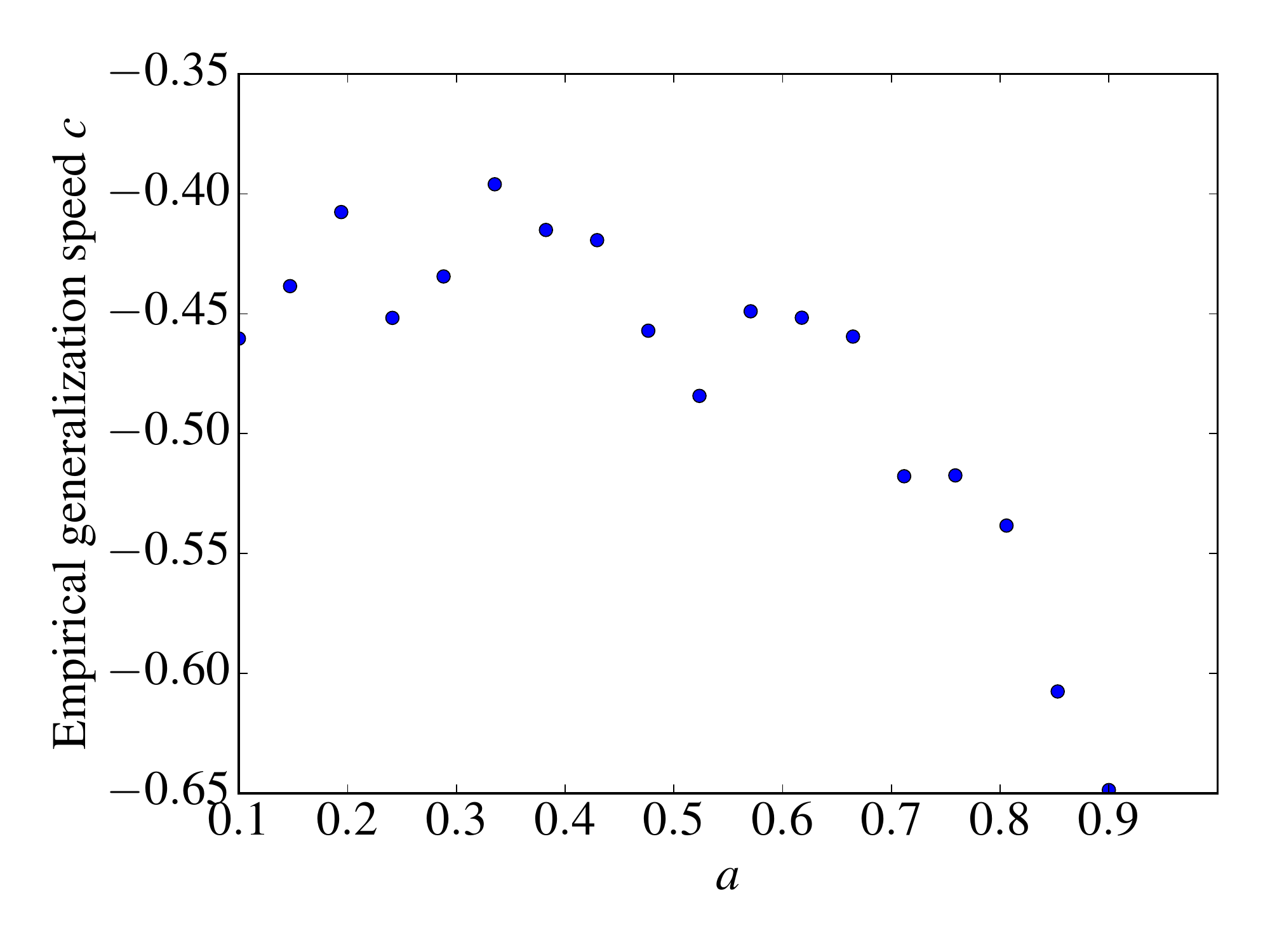}
\caption{Generalization speed for different values of $a$.}
\label{fig:fast_gen_speeds}
\end{figure}

\cref{ccl-fast-rates} shows that when the noise assumption NA is verified,
faster rates of generalization can be achieved. Showing the existence of fast rates experimentally
requires us to design a problem for which the $\eta$ satisfies NA, which is not trivial due to the pairwise nature of the involved quantities. We emphasize that such empirical evidence of fast rates is rarely found in the literature.
 
We put ourselves in a simple scenario where $\X = [0,1]$, $\mu = 1$, $K=2$ and $p_1=p_2=1/2$.
In that context, characterizing $\mu_1(dx)$ is sufficient to have a fully defined problem.
With $m \in (0,\frac{1}{2})$, $a \in (0,1)$
and $C \in (0,\frac{1}{2})$, we set
\begin{align*}
\mu_1(x) = 
\begin{cases}
2 C \quad &\text{if} \quad x \in [0,m], \\
1 - \abs{2x-1}^{(1-a)/a} \quad &\text{if} \quad x  \in ( m, 1/2] ,
\end{cases}
\end{align*}
where $C$ is chosen so that $Q^*_\alpha=1/2$ and $m$ is fixed in advance.
Since $\int \mu_1(dx) = 1$, we chose $\mu_1$ symmetric in $ \left (1/2, 1 \right )$ to satisfy that constraint. 
Figure~\ref{fig:fast_example_distrib} shows example distributions.

Given that $\mu = 1$, the noise assumption with $a$ close to $1$ requires that there are sharp variations of 
$\eta$ close to $Q^*_\alpha$. To induce the form of the function more easily, we fixed $Q^*_\alpha =1/2$, which
requires us to choose $\mu_1$ such that the value of the integral of $\eta$ is controlled while $\eta$ has the 
expected local property around $1/2$. 
 More details about the design of the experiment can be found in the supplementary material. 
 When $t$ is small enough, $\p \left ( \abs{\eta(X,X') -Q^*_\alpha } \le t \right )$ is of order 
 $-t^{\frac{a}{1-a}} \log(t)$. Due to the logarithm term in the noise condition, we expect that the generalization
 speeds to be slightly worse than $O(n^{-(2+a)/4})$.

 The family $\mathcal{S}_0$ is composed of indicators of sets, which are parameterized by $t \in (0,1)$ (see supplementary material for a graphical representation). Each set contains the pairs $(x,x')$ such that one of the supremum
  distances between $(x,x')$ and $(0,0)$ or $(1,1)$ is smaller than $t$, which writes
\begin{align*} 
 \left\{ x,x' \in \X \;  |  \;
 \min( \max(1-x,1-x') , \max(x,x') ) < t  \right\}.
\end{align*} 
The optimal set can thus always be identified, and $R^+(S)$ and $R^-(S)$ can be expressed
analytically for some $S \in \mathcal{S}_0$. The empirical problem \cref{simlearn} is always solved neglecting 
the tolerance parameter $\Phi$, i.e. setting $\Phi = 0$.

Figure~\ref{fig:fast_gen_speeds} shows experiments for the case $\alpha = 0.26$, $m = 0.35$ and $a\in[0.1,0.9]$. 
For some $a$, the empirical 90-quantile of $\roc_{S^*} ( \alpha ) - R^+(\hat{S}_n)$ is computed for different values of 
$n$ on 1000 experiments and its logarithm is fitted to $C_a \times \log(n) + D_a$ to get 
the empirical generalization speed $C_a$. There is a clear downward trend when $a$ increases, illustrating the fast 
rates in practice.

\begin{figure}[t]
\centering
\includegraphics[width=\columnwidth]{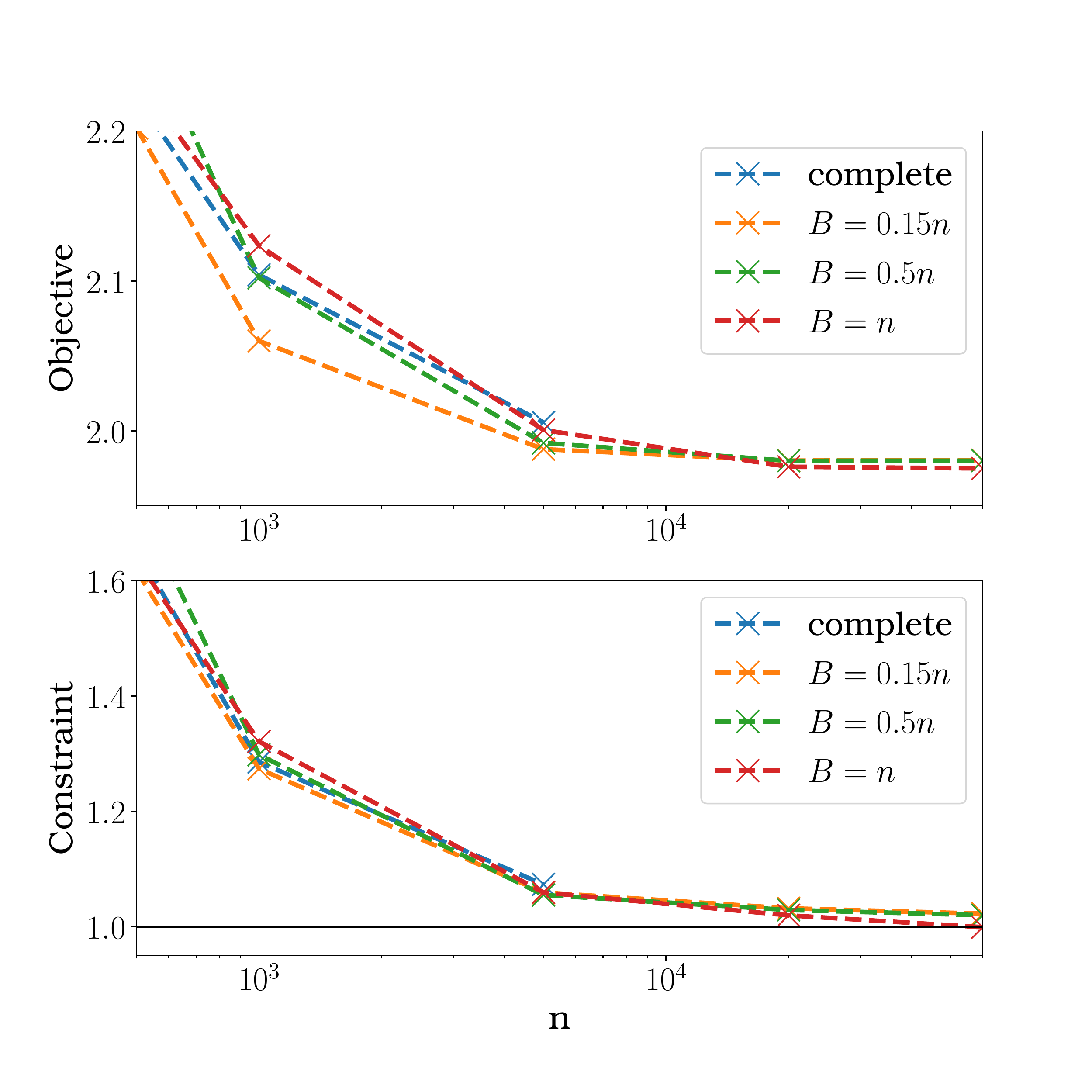}
\caption{Value of objective and constraint on the test set for various levels of approximation of the negative risk, averaged on 5 runs for 
each set of parameters $(n,B)$.}
\label{fig:incomplete_U_stats}
\end{figure}

\subsection{Scalability by Sampling}
We illustrate the results of Section~\ref{sec:scalability} on MMC \citep{Xing2002a}, a popular metric learning algorithm whose formulation is very close to the one we consider. We introduce the set of Mahalanobis distances $d_A$ indexed by a positive semidefinite matrix $A$:
\begin{align*}
d_A(x,x') &=\sqrt{(x-x')^\top A (x-x')}.
\end{align*}  
MMC solves the following problem (using projected gradient ascent):
\begin{align*}
\begin{split}
\max_A \; &\frac{1}{n_-} \sum_{1\le i < j \le n} \I\{ Y_i \ne Y_j \} \cdot d_A(X_i, X_j) \\
\text{s.t.} \; & \frac{1}{n_+} \sum_{1\le i < j \le n} \I\{ Y_i = Y_j\} \cdot d_A^2( X_i, X_j) \le 1\\
& A \succeq 0\\
\end{split}
\end{align*}  
We use MNIST dataset, composed of $70,000$ images representing the 
0-9 handwritten digits, with classes roughly equally distributed.
We randomly split it into a training set and a test set of $10,000$ instances. 
As done in previous work, the dimension of the features is reduced using PCA to keep 90\% of the explained variance.
We approximate the average over negative pairs by sampling $K$-tuples with $B$ terms, as proposed in Section~\ref{sec:scalability} (pair-based sampling performs similarly on this dataset). 
We aim to show that optimizing the criterion on the resulting smaller set of pairs does not significantly impact the learning rate (yet greatly reduces training time). 
We solve MMC on the training set for a varying number of training instances $n$ and of $K$-tuples $B$, and 
report the objective and constraint values on the test set.
The results, summarized in 
Figure~\ref{fig:incomplete_U_stats}, confirm the small performance loss due to subsampling, 
for a huge improvement in terms of computing time. Indeed, when $n = 60,000$, the total number of negative pairs 
is almost $2$ billions while $B = 0.15 n$ corresponds to sampling only $400,000$ pairs.


\section{Conclusion}
\label{conclusion}

We have introduced a rigorous probability framework to study similarity learning from the novel perspective of pairwise bipartite ranking and pointwise ROC optimization. We derived statistical guarantees for generalization in this context, and analyzed the impact of using sampling-based approximations. Our results are illustrated on a series of numerical experiments.
Our study opens promising directions of future work. We are especially interested in extending our results to allow the rejection of queries from unseen classes (e.g., unknown identities) at test time \citep[see for instance][]{Bendale2015a}. This could be achieved by incorporating a loss function to encourage the score of all positive pairs to be above some fixed threshold, below which we would reject the query.


\newpage

\section*{Acknowledgments}

This work was supported by IDEMIA.
We would like to thank Anne Sabourin for her substantial feedback that has greatly improved this work,
as well as the ICML reviewers for their constructive input.

%

\bibliographystyle{icml2018}
\bibliography{Ref_Ranking}

\appendix
\onecolumn

\section*{SUPPLEMENTARY MATERIAL}

\section{Technical proofs}

\subsection{Definitions}
\label{appendix-definitions}

We recall a few useful definitions.

\begin{definition}[VC-major class of functions -- \citealp{VanderVaart1996}]
A class of functions $\mathcal{F}$ such that $\forall f \in \mathcal{F}$, $f:\X \to \R$ is called VC-major
if the major sets of the elements in $\mathcal{S}$ form a VC-class of sets in $\X$.
Formally, $\mathcal{S}$ is a VC-major class if and only if:
\begin{align*}
\left \{  \left \{ x \in \X \mid f(x) > t \right \} \mid f \in \mathcal{F}, t \in \R \right \}
\; \text{is a VC-class of sets.}
\end{align*}
\end{definition}

\begin{definition}[$U$-statistic of degree 2 -- \citealp{Lee90}]
Let $\X$ be some measurable space and $V_1,\; \ldots,\; V_n$ i.i.d. random variables valued in $\X$ and $K:\X^2\rightarrow \mathbb{R}$ a measurable symmetric mapping s.t. $h(V_1,V_2)$ is square integrable. The functional $U_n(h)=(1/n(n-1))\sum_{i\neq j}h(V_i,V_j)$ is referred to as a symmetric $U$-statistic of degree two with kernel $h$. It classically follows from Lehmann-Scheff\'e's lemma that it is the unbiased estimator of the parameter $\mathbb{E}[h(V_1,V_2)]$ with minimum variance.
\end{definition}

\begin{definition}[Generalized $U$-statistic -- \citealp{JMLRincompleteUstats}]
Let $K \ge 1$  and $(d_1, \dots, d_K) \in \N^{*K}$. 
Let $X_{1, \dots, n_k} = \{ X_1^{(k)}, \dots, X_{n_k}^{(k)} \}$, $1 \le k \le K$, be $K$ independent samples of sizes
$n_k \ge d_k$ and composed of i.i.d. random variables taking their values in some measurable space $\X_k$ with 
distribution $F_k(dx)$ respectively. Let $h : \X_1^{d_1} \times \dots \times \X_K^{d_K} \to \R$ be a measurable 
function, square integrable with respect to the probability distribution $\mu = F_1^{\otimes d_1} \otimes \dots 
\otimes F_k^{\otimes d_k}$. Assume in addition (without loss of generality) that $h(x^{(1)} , \dots , x^{(K)})$
is symmetric within each block of arguments $x^{(k)}$ (valued in $\X_k^{d_k}$), $1 \le k \le K$. The generalized
(or $K$-sample) $U$-statistic of degrees $(d_1, \dots, d_K)$ with kernel $h$, is then defined as
\begin{align*}
U_n(h) = \frac{1}{\prod_{k=1}^K \binom{n_k}{d_k}} \sum_{I_1} \dots \sum_{I_K} h \left ( X_{I_1}^{(1)} , 
X_{I_2}^{(2)}, \dots, X_{I_K}^{(K)}  \right ),
\end{align*}
where the symbol $\sum_{I_k}$ refers to summation over all $\binom{n_k}{d_k}$ subsets $X_{I_k}^{k} = 
\left ( X_{i_1}^{(k)}, \dots, X_{i_{d_k}}^{(k)} \right ) $ related to a set $I_k$ of $d_k$ indexes 
$1 \le i_1 < \dots < i_{d_k} \le n_k$ and $n = ( n_1 , \dots, n_K)$.
\end{definition}

\subsection{Proof of \cref{ccl-slow-rates}}

The proof relies on the following argument from \citet[][Theorem~10 therein]{Clemencon2010} which points out that
\begin{align}
\begin{split}
  &\p \left\{\left\{ R^+ (\hat{S}_n)\ge \sup_{S \in \mathcal{S}_0 : \, R^-(S) \le \alpha } R^+(S)  
  -   \Phi_{n,\delta/2} \right\} \cap   \left\{ R^- ( \hat{S}_n )\le \alpha +  \Phi_{n,\delta/2} \right\} \right\}\\
  &\ge 1 - \p\left\{ \sup_{ S \in\mathcal{S}_0 } \abs{R_n^+(S) - R^+(S)} > \Phi_{n,\delta/2}  \right\} 
  - \p\left\{ \sup_{ S \in\mathcal{S}_0 } \abs{R_n^-(S) - R^-(S)} > \Phi_{n,\delta/2}  \right\}.
  \end{split}
  \label{ineg-emp-proc2}
\end{align}

Set $p =\sum_{k=1}^K p_k^2$ and rewrite $\widehat{R}^+_n(S)=\frac{n(n-1)}{2n_+}U^+_n(S)$ and $\widehat{R}^-_n(S)=\frac{n(n-1)}{2n_-}U^-_n(S)$ where $U^+_n(S)$ and $U^-_n(S)$ has $U$-statistics. Observe that we have: $\forall n>1$,
\begin{align*}
\sup_{S \in \Scal_0}\abs{ \hat{R}_n^+ (S) - R^+(S) }
&  = \sup_{S \in \Scal_0} \abs{ \left ( \frac{n(n-1)}{2n_+} - p^{-1} \right ) U_n^+(S) +  p^{-1}\left ( U_n^+(S) - \E [U_n^+(S)] \right ) }, \\
&\le  \abs{  \frac{n(n-1)}{2n_+} - p^{-1}  } + p^{-1} \sup_{S \in \Scal_0} \abs{  U_n^+(S) - \E [U_n^+(S)] }, \\
&\le  p^{-1} \frac{\abs{ 2n_+/n(n-1) -  p }}{ p + 2n_+/n(n-1) - p  }  + p^{-1} \sup_{S \in \Scal_0} \abs{  U_n^+(S) - \E [U_n^+(S)] }. \\
\end{align*}
\citet[Section 5.6, Theorem A][]{Ser80} gives that, with probability at least $1-\delta$,
\begin{align*}
\abs{2n_+/n(n-1) - p}  < \sqrt{\frac{\log (2/\delta)}{n-1}}.
\end{align*}
Hence when $2n_+/n(n-1) - p \ge 0$, we have:
\begin{align*}
\frac{\abs{ 2n_+/n(n-1) -  p }}{ p + 2n_+/n(n-1) - p  } \le p^{-1} \sqrt{\frac{\log (2/\delta)}{n-1}}.
\end{align*}
And when $2n_+/n(n-1) - p < 0$, we have, for $n \ge 1 + 4p^{-2} \log (2/\delta)$:
\begin{align*}
\frac{\abs{ 2n_+/n(n-1) -  p }}{ p + 2n_+/n(n-1) - p  } &\le \frac{\sqrt{\frac{\log(2/\delta)}{n-1}}}{p - p/2 }, \\
&\le 2p^{-1} \sqrt{\frac{\log(2/\delta)}{n-1}}.
\end{align*}
Combining this proposition with Lemma 1 gives that: with probability at least $1-\delta$,  
$ \forall n \ge 1 + 4p^{-2} \log(3/\delta)$,
\begin{align*}
\sup_{S \in \Scal_0}\abs{ \hat{R}_n^+ (S) - R^+(S) }
&\le  2 c p^{-1} \sqrt{\frac{V}{n}} + 2p^{-1}(1+p^{-1}) \sqrt{\frac{\log(3/\delta)}{n-1}} . \\
\end{align*}

Similarly, we obtain, with $q=1-p$:
\begin{align*}
\sup_{S \in \Scal_0}\abs{ \hat{R}_n^- (S) - R^-(S) }
&\le q^{-1} \frac{\abs{ 2n_-/n(n-1) -  q }}{ q + 2n_-/(n(n-1)) - q }  + q^{-1} \sup_{S \in \Scal_0} \abs{  U_n^-(S) - \E [U_n^-(S)] }, \\
\end{align*}
which gives with the exact same reasoning that: with probability at least $1-\delta$, $ \forall n \ge 1 + 4 q^{-2} \log(3/\delta)$,
\begin{align*}
\sup_{S \in \Scal_0}\abs{ \hat{R}_n^- (S) - R^-(S) }
&\le  2 c q^{-1} \sqrt{\frac{V}{n}} + 2q^{-1}(1+q^{-1}) \sqrt{\frac{\log(3/\delta)}{n-1}}, 
\end{align*}
which gives the right order of convergence. The proof is then finished by following the proof of Theorem 10 in \cite{Clemencon2010}. 


\subsection{Proof of \cref{ccl-fast-rates}}
Let $S^*(x,x')=\mathbb{I}\{(x,x')\in\mathcal{R}^*_{\alpha} \}$ be the optimal similarity function (supposed here to belong to $\mathcal{S}_0$ for simplicity).
Observe that, for any $S\in \mathcal{S}$, the statistic 
\begin{equation}\label{eq:deltaU}
\Delta_{n} (S) = \left ( U_n^+(S) - \E \left [ U_n^+(S) \right ] \right ) 
- \left ( U_n^+(S^*) - \E \left [ U_n^+(S^*) \right ] \right ),
\end{equation}
is a $U$-statistic based on $\mathcal{D}_n$ with kernel $Q_S$ given by
\begin{align*}
Q_S((x,y),(x',y')) &= \I\left \{ y=y' \right \} \left ( S(x,x')-S^*(x,x')\right ) 
- \E \left [ \I\left \{ Y=Y' \right \} \left ( S(X,X') - S^*(X,X') \right ) \right ] .
\end{align*}
The second Hoeffding decomposition of $U$-statistics, already used for the fast rate analysis carried out in
\cite{Clemencon08Ranking} in ranking risk minimization, leads to the following decomposition of the $U$-statistic
$\Delta_n(S)$:
\begin{equation}\label{eq:decomp}
\Delta_{n}(S) = 2 T_n(S)  + W_n(S) ,
\end{equation}
where 
\begin{equation*}
T_n (S)= \frac{1}{n} \sum_{i=1}^n q_S (X_i,Y_i) \text{ and } W_n(S) = \frac{2}{n(n-1)} \sum_{1 \le i < j \le n} \widehat{q}_S \left ((X_i, Y_i), (X_j , Y_j) \right ) ,
\end{equation*}
with 
\begin{align*}
q_S (x,y) &= \E \left [ Q_S\left ( (X, Y), (x, y) \right ) \right ], \\
\widehat{q}_S \left ( (x,y),(x',y') \right ) &= 
Q_S \left ( (x, y) , (x',y' ) \right ) - q_S (x,y) - q_S (x',y').
\end{align*}
The component $W_n(S)$ is a degenerate $U$-statistic, meaning that, for all $(x,y)$, $$\mathbb{E}\left[\widehat{q}_S \left ( (x,y),(X,Y) \right )\right]=0 \text{ almost-surely.}$$
The following lemma is thus a direct consequence of the bound established in \cite{ArconesGine}.
\begin{lemma}
Suppose that the hypotheses of Theorem \ref{ccl-fast-rates} are fulfilled. 
Then, for any $\delta\in (0,1)$, we have with probability larger than $1-\delta$: $\forall n\geq 2$,
 $$
 \sup_{S\in \mathcal{S}_0}\left\vert W_n(S) \right\vert \leq C\times \frac{V\log(n/\delta)}{n},
 $$
 where $C<+\infty$ is a universal constant.
 \label{arcones-gine}
 \end{lemma}
 This result shows that the second term in decomposition \eqref{eq:decomp} is uniformly negligible with respect to $T_n(S)$. The final ingredient is the following lemma, which provides a control of the variance of $T_n(S)$ under the {\bf NA} condition.
 \begin{lemma} Suppose that the hypotheses of \cref{ccl-fast-rates} are fulfilled. Then, 
 for any $S\in \mathcal{S}_0$, we have:
 \begin{equation}
\text{Var} (q_S(X,Y)) 
\le  c \left [ (1-Q^*_\alpha) p ( R^+ (S^*) - R^+ (S) ) + (1-p) Q^*_\alpha (R^-(S) - R^-(S^*) ) \right ]^a. 
\label{var-control-1-eq}
 \end{equation}
 \label{var-control-1}
 \end{lemma}
\begin{proof}

We introduce $\ominus$ as the symmetric difference operator between two sets.
Introducing
$\Rcal_{S,u} = \{x,x' \mid S(x,x') \ge u\}$ and $\Rcal_\alpha^* = \{x,x' \mid \eta(x,x') > Q^*_\alpha \}$, we have that:
\begin{align*}
\text{Var} \left ( q_S(X,Y) \right ) & \le \E_{X,Y} \left [ \left ( \E_{X',Y'} 
\left [ \I\{Y = Y' \} (S(X,X') - S^*(X,X'))\right ] \right )^2  \right ],  \\
&\qquad (\text{By the second moment bound}), \\
& \le \int_0^1 \E_{X} \left [ \left ( \E_{X'} 
\left [   \I\{Y = Y' \} (\I\{ (X,X') \in \Rcal_{S,u} \} - \I\{ (X,X') \in \Rcal_\alpha^* \} ) \right ] \right )^2  
\right ] du, \\
& \qquad \left ( \text{Since } S(X,X') = \int_0^1 \I\{  (X,X') \in \Rcal_{S,u} \} \, du \text{ and } 
S^*(X,X') = \I \{(X,X') \in \Rcal_\alpha^* \},  \right ) \\
& \le \int_0^1 \E_{X} \left [ \left ( \E_{X'} 
\left [ \I\{  (X,X') \in \Rcal_{S,u} \ominus \Rcal_\alpha^*  \}  \right ] \right )^2  \right ] du, \\
& \le \int_0^1 \E_{X} \left [  
\E_{X'} \left [ \I\{  (X,X') \in \Rcal_{S,u} \ominus \Rcal_\alpha^*  \} \abs{\eta(X,X')-Q^*_\alpha}^{a} \right ] 
\E_{X'} \left [ \abs{\eta(X,X')-Q^*_\alpha}^{-a} \right ]
\right ] du, \\
&\qquad (\text{By Cauchy-Schwartz's inequality}), \\
& \le c  \left ( \int_0^1 \E_{X} \left [  
\E_{X'} \left [ \I\{  (X,X') \in \Rcal_{S,u} \ominus \Rcal_\alpha^*  \} \abs{\eta(X,X')-Q^*_\alpha} \right ] 
\right ] du \right )^a , 
\addtocounter{equation}{1}\tag{\theequation} \label{variance-bound} \\
& \qquad  \text{(By the beforementioned noise hypothesis followed by Jensen's inequality)}.
\end{align*}

The right-hand-side of \cref{variance-bound} is very similar to an expression of the excess risk in binary 
classification, see \citet[][Eq~1 therein]{Boucheron2005}. We now link it to the right-hand side of 
\cref{var-control-1-eq}.

Formally, note that:
\begin{align*}
R^+(S)-R^+(S^*) &= p^{-1} \E \left [ \I\{Y=Y'\}(S(X,X') - S^*(X,X')) \right ], \\
&= p^{-1} \int_0^1 \E \left [ \eta(X,X')(\I\{  (X,X') \in \Rcal_{S,u} \} - \I\{ (X,X') \in \Rcal_\alpha^* \} ) \right ] du, \\
&= p^{-1} \int_0^1 \E \left [ \eta(X,X')\I\{  (X,X') \in \Rcal_{S,u} \ominus \Rcal_\alpha^*  \} 
\left (  1 - 2 \I\{ (X,X') \in \Rcal_\alpha^* \}  \right ) \right ] du, \\
&= - p^{-1} \int_0^1 \E \left [ \abs{\eta(X,X')-Q^*_\alpha}\I\{  (X,X') \in \Rcal_{S,u} \ominus \Rcal_\alpha^*  \} \right ] du +
 p^{-1} Q^*_\alpha \E \left [S(X,X') - S^*(X,X) \right ] du, \\
 & \qquad \left ( \text{Since } \left ( 1-2\I\{(X,X') \in \Rcal_\alpha^* \} \right ) \left ( \eta(X,X')-Q^*_\alpha \right )  
 = - \abs{\eta(X,X')-Q^*_\alpha } .  \right )
\end{align*}
which implies that
\begin{align}
\int_0^1 \E \left [ \abs{\eta(X,X')-Q^*_\alpha}\I\{  (X,X') \in \Rcal_{S,u} \ominus \Rcal_\alpha^*  \} \right ] du &=
(1-Q^*_\alpha)p(R^+(S^*) - R^+(S)) + Q^*_\alpha (1-p) (R^-(S)-R^-(S^*)).
\label{excess-risk-reformulation}
\end{align}

Combining \cref{variance-bound} and \cref{excess-risk-reformulation} completes the proof.

\end{proof} 
 
 \cref{var-control-1} is the analogue of Lemma 11 in \cite{Clemencon2010} and the fast rate bound stated in Theorem \ref{ccl-fast-rates} then classically follows from the application of Talagrand's inequality (or Bernstein's inequality when $\mathcal{S}_0$ is of finite cardinality), see \textit{e.g.} subsection 5.2 in \cite{Boucheron2005} and Theorem 12 in \cite{Clemencon2010}.

We state here the proof for the case where $\mathcal{S}_0$ is of finite cardinality.
Proving this result for more general classes of functions $\Scal_0$ is tackled by the localization argument expressed in 
\citet[][pages 341-346 therein]{Boucheron2005} --- we omit it to avoid stretching the proof unnecessarily.

Since $\abs{Q_S} \le 2$, Bernstein's inequality gives that: 
for all $S\in\mathcal{S}_0$, with probability at least $1-\delta$,
\begin{align*}
T_n(S) \le \frac{4\log (1/\delta) }{3n} + \sqrt{ \frac{ 2\text{Var} (q_S(X,Y)) \log (1/\delta ))}{ n} }.
\end{align*}
If $\mathcal{S}_0$ is of cardinality $M$, it implies that: with probability at least $1-\delta$,
for all $S \in\mathcal{S}_0$,
\begin{align}
T_n(S) \le \frac{4\log (M / \delta) }{3n} + \sqrt{ \frac{ 2\text{Var} (q_S(X,Y)) \log ( M / \delta ))}{ n} }.
\label{union-bound-Tn}
\end{align}
An equivalent of \cref{arcones-gine} for the case of finite classes of functions $\mathcal{S}_0$ can be derived from 
\citet[][Theorem~4.1.13 therein]{PenaG99}, which is that: with probability at least $1-\delta$,
\begin{align}
\sup_{S \in\mathcal{S}_0 }
\abs{W_n(S)} \le \frac{2C \log (4 M / \delta)}{n},
\label{union-bound-Wn}
\end{align}
since $W_n(S)$ is a degenerate $U$-statistic. $C$ is a universal constant.

Combining \cref{union-bound-Tn} and \cref{union-bound-Wn} give that: with probability at least $1-\delta$, 
for all $S \in\mathcal{S}_0$
\begin{align}
\Delta_n(S) \le \frac{C'\log (5 M / \delta) }{n} 
+ \sqrt{ \frac{ 2\text{Var} (q_S(X,Y)) \log ( 5 M / \delta ))}{ n} }.
\label{bound-delta-n-uniform}
\end{align}
where $C' = 2C + 4/3$.

The proof of \cref{ccl-slow-rates} may be adapted to the finite class setting. Formally, introducing the tolerance
term:
\begin{align*}
\Phi_{n,\delta}' = 2 \kappa^{-1} (1+\kappa^{-1})
\sqrt{\frac{\log(2(M+1)/\delta)}{n-1}},
\end{align*}
where $N$ is the cardinal of the proposition class $\mathcal{S}_0$, we have with probability $1-\delta$,
\begin{align}
R^+ (\hat{S}_n) &\ge \sup_{S\in \mathcal{S}_0:\; R^-(S)\leq \alpha} R^+(S)  - \Phi_{n,\delta/2}'
\quad \text{and} \quad 
R^-(\hat{S}_n) \le \alpha + \Phi_{n,\delta/2}', \label{finite-theorem-1} \\
&\sup_{S \in \mathcal{S}_0} \abs{R_n^+(S)-R^+(S)} \le \Phi_{n,\delta/2}'. \label{sup-control}
\end{align}

\cref{sup-control} implies that $S^*$ satisfies the constraint of the ERM problem \cref{simlearn}, 
hence $R_n^+(\hat{S}_n) - R_n^+(S^*) \ge 0$. It follows that :
\begin{align*}
\Delta_n(\hat{S}_n) &= \frac{2n_+}{n(n-1)} \left ( R_n^+(\hat{S}_n) - R_n^+(S^*) \right ) 
+ p \left ( R^+(S^*) - R^+ (\hat{S}_n) \right ),\\
&\ge  p \left ( R^+(S^*) - R^+ (\hat{S}_n) \right ).
\addtocounter{equation}{1}\tag{\theequation} \label{Delta-n-lower-bound} 
\end{align*}

Let $\hat{S}_n$ be the solution of \cref{simlearn} with $\Phi_{n,\delta'/2}'$, 
where $\delta' = 2(M+1)\delta /(9M+4)$.
Introducing $ K_{\delta,M} = ( 9 M + 4 )/ \delta$, we combine 
\cref{var-control-1}, \cref{bound-delta-n-uniform}, \cref{finite-theorem-1} and \cref{Delta-n-lower-bound},
to obtain that: with probability at least $1-\delta$,
\begin{align}
p \left ( R^+(S^*) - R^+(\hat{S}_n) \right ) \le \frac{C'\log K_{\delta,M} }{n} 
+ \sqrt{ \frac{ 2 c \log K_{\delta,M} }{ n} } \left ( 
[(1-Q^*_\alpha) p (R^+(S^*) - R(\hat{S}_n)) ]^{a/2} 
+  [Q^*_\alpha (1-p) \Phi_{n,\delta'/2}' ]^{a/2}  \right ).
\label{End-bound} 
\end{align}
The highest order term on the right-hand side is in $O ( n^{-1/2}\Phi_{n,\delta'/2}'^{a/2} )$ which is $O(  n^{-(2+a)/4})$.
\cref{End-bound} is a fixed point equation in $R^+(S^*) - R^+ (\hat{S}_n)$.
Finding an upper bound on the solution of this fixed-point equation is done by invoking \citet[][Lemma~7 therein]{Cucker2002}, which we recall here for completion.
\begin{lemma}
Let $c_1, c_2 > 0$ and $s > q > 0$. Then the equation
$x^s - c_1 x^q - c_2 = 0,$ has a unique positive zero $x^*$.
In addition, $x^* \le \max \{ (2c_1)^{1/(s-q)}, (2c_2)^{1/s} \}$.
\label{cucker-lemma}
\end{lemma}

Applying \cref{cucker-lemma} to \cref{End-bound} concludes the proof.

\subsection{Remarks on the noise assumption NA}
\label{NA-remark}

We now discuss the type of conditions under which our noise assumption is true.
Assume that there exist $A \subset \X, \p(X\in A ) = 1$, such that for all $x \in A$, the random variable 
$\eta(x,X)$ has an absolutely continuous distribution on $[0,1]$ 
and its density is bounded by $B$. Intuitively, this assumption means that the problem of ranking elements modeled by $X$ according to their similarity with 
an element $x\in A$ is somewhat easy (almost-surely).

In this case, we have that: for any $\epsilon >0$,
\begin{align*}
\quad \E_{X'} \left ( \abs{\eta(X,X')-Q^*_\alpha}^{-1+ \epsilon} \right )  \le \frac{2B}{\epsilon} \quad \text{almost surely},
\end{align*}
which implies that the fast rate of convergence of \cref{ccl-fast-rates} applies for any $a \in (0,1)$.
The proof is given below and follows the same arguments as \citet[][Corollary~8 therein]{Clemencon08Ranking}.

\begin{proof}
Let $x \in A$ and $h_x$ be the density of $\eta(x,X)$, with $h_x \le B$. Hence,
for any $a \in (0,1)$,
\begin{align*}
\E_{X'} \left ( \abs{\eta(x,X)-Q^*_\alpha}^{a} \right ) &=
\int_{0}^1 \abs{z-Q^*_\alpha}^{a}  h_x(z) dz, \\
&\le B \left( \int_{0}^{Q^*_\alpha} (Q^*_\alpha-z)^{a} dz + 
\int_{Q^*_\alpha}^1 (z-Q^*_\alpha)^{a}  dz  \right ), \\
&\le B \left( \frac{Q^{*1+a}_\alpha}{1+a} + 
\frac{(1 - Q^{*}_\alpha)^{1+a}}{1+a} \right ), \\
&\le \frac{2B}{1+a}.
\end{align*}
\end{proof}


\subsection{Proof of of \cref{prop:var}}

Conditioned upon the $(Y_i)_{i=1}^n$, the sample $\left ( X_i, Y_i \right )_{i=1}^n$ is seen as $K$ samples 
$\left ( X_{i_k}^{(k)} \right )_{i_k=1}^{n_k}$, $k \in \{ 1, \dots , K\}$, one for each class, and
$R_n^-(S)$ can be written as a $K$-sample generalized $U$-statistic with kernel $h_S$. Indeed,
\begin{align}
\begin{split}
R_n^-(S) &= \frac{1}{n_-} \sum_{k<l} 
\sum_{ i_{k} = 1}^{n_k} \sum_{i_l=1}^{n_l} S \left ( X_{i_k}^{(k)}, X_{i_l}^{(l)} \right ), \\
 &=  \frac{1}{\prod_{k=1}^K n_k }\sum_{ i_{1} = 1}^{n_1}  \dots  \sum_{ i_{K} = 1}^{n_K}  
   \sum_{k<l}  \frac{n_k n_l}{n_-} S \left ( X_{i_k}^{(k)}, X_{i_l}^{(l)} \right ), \\
    &=  \frac{1}{\prod_{k=1}^K n_k }\sum_{ i_{1} = 1}^{n_1}  \dots  \sum_{ i_{K} = 1}^{n_K}  
    h_S \left ( X_{i_1}^{(1)} , \dots, X_{i_K}^{(K)} \right ).\\
\end{split}
\label{alternative-view-neg-risk}
\end{align}

From \citet[][Equation~(21) therein]{JMLRincompleteUstats}, we have that
\begin{align*}
\text{Var}\left ( \widetilde{R}_B^- (S) \right ) &= 
\left( 1 - \frac{1}{B} \right) \text{Var} \left ( R_n^-(S)\right ) 
+ \frac{1}{B} \text{Var} \left ( h_S \left ( X^{(1)} , \dots , X^{(K)} \right ) \right ),
\end{align*}
which gives the result for tuple-based sampling since 
$\text{Var}\left ( R_n^-(S) \right ) = O \left ( 1 / n \right )$ when $B_0/n \to 0$, $n\to \infty$
 and $n_k/n \to p_k > 0$ for all $k\in\{1,\dots, K \}$. 

 
Straightforwardly adapting the proof of \citet[][Equation~(21) therein]{JMLRincompleteUstats} for the case of 
pair-based sampling gives:
\begin{align*}
\text{Var} \left ( \bar{R}_B^-(S) \right )  &=  \left ( 1 - \frac{1}{B} \right ) \text{Var} \left ( R_n^-(S) \right )
+ \frac{1}{B} \left [  \sum_{k<l}  \frac{n_k n_l}{n_-} \E \left [S^{2} \left ( X^{(k)}, X^{(l)} \right )\right ] 
- \left ( \sum_{k<l} \frac{n_k n_l}{n_-} \E \left [ S \left ( X^{(k)}, X^{(l)} \right )\right ] \right )^2 \right ],
\\
 &=  \left ( 1 - \frac{1}{B} \right ) \text{Var} \left ( R_n^-(S) \right )
+ \frac{1}{B} \text{Var} \left ( S \left ( X,X' \right ) \, | \, Y \ne Y' \right ).  \\
\end{align*}

\subsection{Proof of \cref{thm:incomplete}}
The proof of \cref{thm:incomplete} is based on an equivalent of \cref{ccl-slow-rates}, in the case where we 
condition upon the labels. It boils down to studying the same problem with $K$ independent samples of i.i.d. data, 
one for each class.
Each sample is written $(X_i^{(k)})_{i=1}^{n_k}$ for all $k \in \{1,\dots, K\}$.

\begin{theorem}
Let $\alpha \in (0,1)$, assume that $S^* \in \mathcal{S}_0$ and that $\mathcal{S}_0$ is a VC-major class of VC-dimension $V$. Let $N = \min_{k \in \{1, \dots, K\} } n_k$.
For all $(\delta, n) \in (0,1) \times \N^*$, set 
\begin{align*}
\Phi_{n,\delta} &= 2 \sqrt{\frac{2V}{N-1}} + \sqrt{\frac{\log(2/\delta)}{N}},
\end{align*}
then we have simultaneously, with probability at least $1-\delta$,
\begin{align*}
R^+(\hat{S}_n) \ge R^+_* - 2 \Phi_{n,\delta} \qquad \text{and} \qquad R^-(\hat{S}_n) \le \alpha + 2 \Phi_{n,\delta}.
\end{align*}
\label{ccl-slow-rates-conditional}
\end{theorem}

\cref{ccl-slow-rates-conditional} is proven by controlling the tail of the supremum over $\mathcal{S}_0$ of the absolute deviation
of $R_n^+$ and $R_n^-$ around their respective means, which is adressed by the two following \cref{pos-process-conditional} and \cref{neg-process-conditional}.

\begin{lemma}
Assume that $\Scal_0$ is a VC-major class of VC-dimension $V$. Let $c$ be a universal constant.
With probability at least $1-\delta$,
\begin{align*}
\sup_{S \in \mathcal{S_0}} \abs{R_n^+(S)-R^+(S)} \le \sum_{k=1}^K \frac{n_k(n_k-1)}{2n_+} 
\left ( 2c \sqrt{\frac{V}{n_k}} + \sqrt{\frac{\log(1/\delta)}{n_k}} \right ).
\end{align*}
\label{pos-process-conditional}
\end{lemma}

The proof of \cref{pos-process-conditional} is based on viewing $R_n^+(S)$ as a weighted average of independent one-sample $U$-statistics of degree two, which writes:
\begin{align*}
R_n^+(S) &= \sum_{k=1}^K \frac{n_k(n_k-1)}{2n_+} \left (  \frac{2}{n_k(n_k-1)} \sum_{ i<j } S \left (X_i^{(k)}, X_j^{(k)} \right ) \right ).
\end{align*}
Chernoff's bound allows us to take advantage of the independence of the $U$-statistics. 
The end of the proof is similar to the one of \citet[][Corollary~3 therein]{Clemencon08Ranking}.

\begin{lemma}
\begin{align*}
\sup_{S \in \mathcal{S_0}} \abs{R_n^-(S)-R^-(S)} \le  2 c \sqrt{\frac{V}{N}} + \sqrt{\frac{\log(1/\delta)}{N}}.
\end{align*}
\label{neg-process-conditional}
\end{lemma}
The proof of \cref{neg-process-conditional} is based on viewing $R_n^-(S)$ as a $K$-sample generalized $U$-statistic, see \cref{alternative-view-neg-risk}.
\cref{neg-process-conditional} then follows from \citet[][Proposition~2 therein]{JMLRincompleteUstats}.

\cref{thm:incomplete} then follows the same steps as \citet[][Theorem~6 therein]{JMLRincompleteUstats}, which gives a tail bound 
on the supremum of $\abs{ \widetilde{R}_B^- - \hat{R}_n^- } $ over $\mathcal{S}_0$. Combining the bound with \cref{pos-process-conditional}, \cref{neg-process-conditional} and
\cref{ineg-emp-proc2} gives the final result.


\section{Experiments}

\subsection{Pointwise ROC optimization}
We define $\mathcal{S}_0$ as the set of bilinear similarities with norm-constrained matrices, i.e.
\begin{align*}
\mathcal{S}_0 = \left \{ S_A : x, x' \mapsto \frac{1}{2} \left ( 1 + x^\top A x' \right ) \; \big | \; 
\norm{A}_F^2 \le 1 \right \},  
\end{align*}
then \cref{simlearn} is written
\begin{align*}
\max_A \; & \frac{1}{n_+} \sum_{1\le i < j \le n} 
\I\left \{ Y_i = Y_j \right \} \cdot S_A \left ( X_i,X_j \right ), \\
\text{s.t.} \; & \frac{1}{n_-} \sum_{1\le i < j \le n} \I\left \{ Y_i \ne Y_j \right \} \cdot
S_A \left ( X_i, X_j \right ) \le \alpha, \\
& ||A||_F^2 \le 1,
\end{align*}
which is equivalent to, with $\beta = 2 \alpha - 1$, 
\begin{align*}
\min_A \; & - \langle P, A \rangle,  \\
\text{s.t. } \; &  \langle N, A \rangle \le \beta, \\
&  \langle A, A \rangle \le 1.
\end{align*}

The problem is always feasible when $\beta \ge 0$. When $\beta < 0$, it is feasible when $\beta \ge - \norm{N}_2 $.
Introducing $(\lambda, \gamma ) \in \mathbb{R}_+$ as the dual variables, the KKT conditions for this problem are
\begin{itemize}
\item Stationarity: $- P + \lambda N + 2 \gamma A = 0$,
\item Primal feasibility: $\langle N,A \rangle \le \beta $ and $\langle A,A \rangle \le 1$,
\item  Dual feasibility: $\lambda \ge 0$ and $\gamma \ge 0$,
\item Complementary slackness: $\lambda ( \langle N,A \rangle - \beta) = 0 $ and 
$ \gamma( \langle A,A \rangle - 1) = 0$,
\end{itemize}
which is solved by considering several cases, specifically,
\begin{itemize}
\item When $P=0$, then $\lambda = 0, \gamma=0$ and any feasible solution suits.
\item When $P$ and $N$ are positively colinear, $\gamma=0, \lambda > 0$ and
any feasible matrix such that $\langle A, A\rangle < 1$, $\langle N , A\rangle = \beta$ suits. 
\item When $\langle N, P \rangle < \beta \norm{P}_2$, $\gamma>0, \lambda = 0$ and the solution is 
$A = P / \norm{P}_2  $
\item Otherwise $\gamma >0, \lambda > 0$ and $\lambda, \gamma, A$ are solutions of
\begin{align*}
- P + \lambda N + 2 \gamma A = 0, \quad \langle N,A \rangle = \beta, \quad  \norm{A}^2_F = 1, \\
\end{align*}
which implies solving a quadratic equation in $\lambda$ when $\beta \ne 0$, and a linear system of equations otherwise.
\end{itemize}

\subsection{Fast rates}
In this section, we justify the choices made in the design of the experiments of \cref{experiments-fast-rates}.
Firstly, when $X$ is uniform on $\X = [0,1]$, i.e. $\mu=1$,
\begin{align}
\p \left( \abs{\eta(X,X') - Q^*_\alpha } \le t  \right) &=
\lambda \left ( \eta^{-1} \left ( \left [ Q^*_\alpha -t, Q^*_\alpha + t \right ] \right )  \right),
\label{noise-cond-unif}
\end{align}
where $\lambda$ is the Lebesgue measure.
Since in a two classes scenario, we have that
\begin{align*}
\mu &= p_1 \mu_1 + p_2 \mu_2, \\
\eta(x,x') &= \left ( p_1^2 \mu_1(x)\mu_1(x') + p_2^2 \mu_2(x)\mu_2(x') \right)/ \left (\mu(x) \mu(x')\right),
\end{align*}
we may explicit $\eta$. Specifically, when $p_1 = p_2 = 1/2$,
\begin{align}
\begin{split}
\eta(x,x') &= \frac{1}{4} \mu_1(x)\mu_1(x') + \frac{1}{4} \left ( 2 - \mu_1(x)\right ) \left ( 2 - \mu_1(x')\right ),\\
&= \frac{1}{2} + \frac{1}{2} \left ( \mu_1(x) - 1 \right )\left( \mu_1(x') - 1 \right ).\\
\end{split}
\label{two-classes-eta}
\end{align}
\cref{noise-cond-unif} and \cref{two-classes-eta} show that if $Q^*_\alpha = 1/2$ and $\mu_1$ varies rapidly close
 to the points $x$ where $\mu_1(x) = 1$, we may obtain fast speeds. To assure that $p_1 = 1/2$, we choose the graph 
 of $\mu_1$ to be symmetric in the point
$(\frac{1}{2}, 1)$. It implies that if a point $x \ge 1/2$ satisfies $\mu_1(x) = 1$, so does $1-x$.
Yet $\mu_1$ has to satisfy very specific local properties in the neighborhood of all of those points.
Hence we choose that $x=1/2$ is the only point satisfying $\mu_1(x) = 1$.

Drawing inspiration from the Mammen-Tsybakov noise condition, we solve for all $ t  \in \left [  0, \frac{1}{2} \right ] $,
\begin{align*}
 t^{\frac{1-a}{a}}  &= \mu_1\left ( \frac{1+t}{2} \right )  - \mu_1\left ( \frac{1-t}{2}  \right ), \\
 t^{\frac{1-a}{a}}  &= 2  - 2 \mu_1\left ( \frac{1 - t}{2}  \right ), \\
\mu_1\left ( \frac{1-t}{2} \right ) &= 1 - \frac{1}{2} t^{\frac{1-a}{a}}, \\
\end{align*} 
which gives that for all $x \in \left[ 0,\frac{1}{2} \right]$, $\mu_1(x) = 1- \frac{1}{2} (1-2x)^{\frac{a}{1-a}}$.

However, choosing this function would conflict with the condition $Q^*_\alpha = \frac{1}{2}$, which requires that
\begin{align}
\label{quantile-condition}
\int_{1-\eta(x,x') > \frac{1}{2}} (1-\eta(x,x'))dxdx' &= \frac{\alpha}{2}.
\end{align}
Therefore, we design a function $\mu_1$ that has the same local properties around $\frac{1}{2}$ as the above function 
but such that the condition on $Q^*_\alpha$ is verified. For that matter, we introduce variables 
$C \in \left ( 0, \frac{1}{2} \right) ,m \in \left (0,\frac{1}{2} \right ) $ such that 
\begin{align*}
\mu_1(x) = 
\begin{cases}
2C \quad &\text{if} \quad x \in [0,m], \\
1 - \abs{1-2x}^{(1-a)/a} \quad &\text{if} \quad x  \in ( m, 1/2],
\end{cases}
\end{align*}
where $m$ is fixed and $C$ is adjusted to meet \cref{quantile-condition}, 
which boils down to $C$ being the solution of a second degree equation.
Solving it gives
\begin{align*}
C = \frac{1}{2} - \frac{\sqrt{1-2\alpha}}{4m} + \frac{a(1-2m)^{a^{-1}}}{4m}.
\end{align*}
For $C$ to satisfy $0 < C < \frac{1}{2}$, the variables $m, \alpha, a$
need to be restricted, as shown by \cref{fig:3}. 
\begin{figure}
\centering
\includegraphics[width=\textwidth]{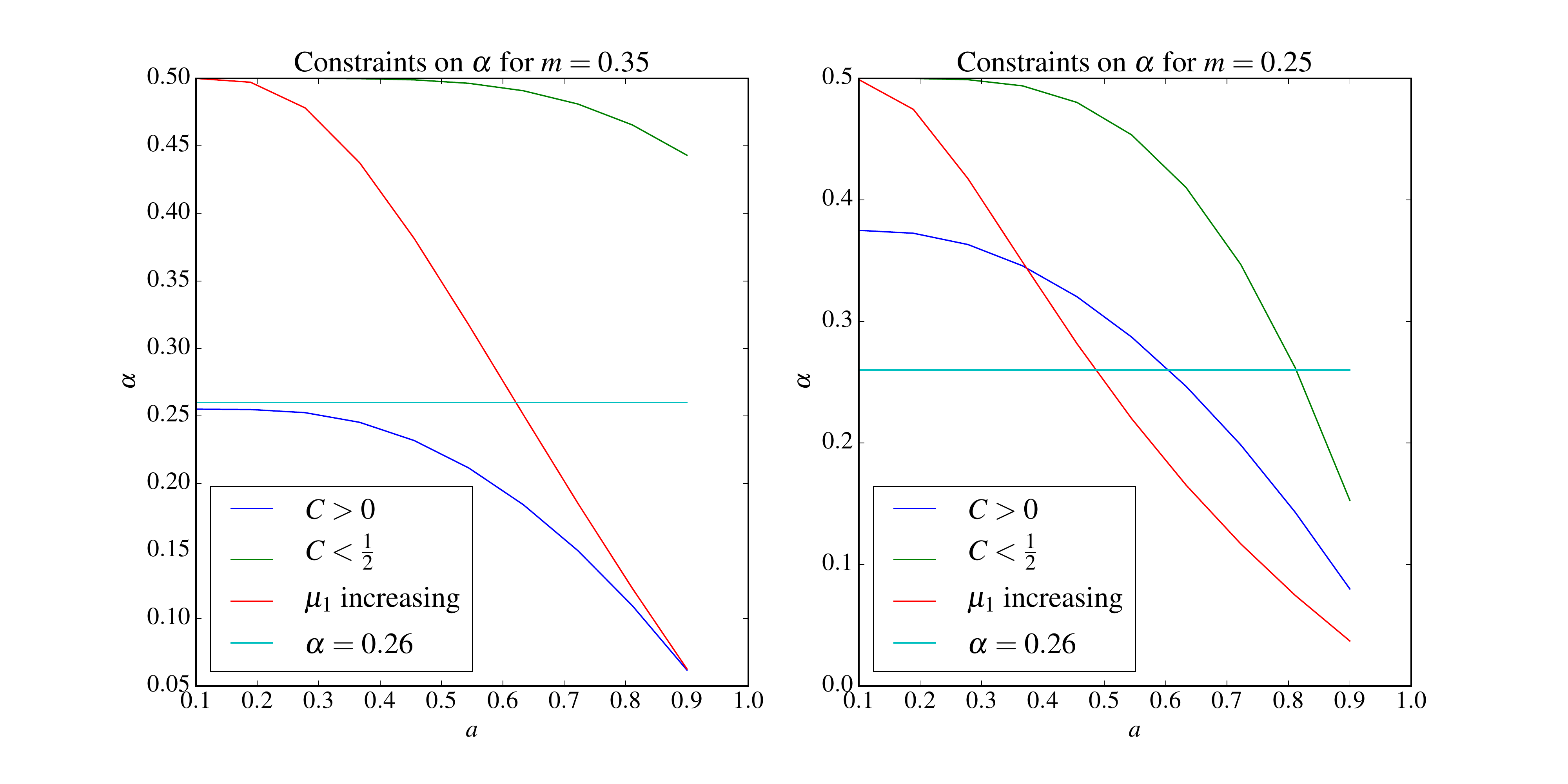}
\caption{Constraints on the value of $\alpha$ for several values of $a$ and two values of $m$.}
\label{fig:3}
\end{figure}
Experimental parameters $(m, a, \alpha)$ are valid if their corresponding point is below the green curve 
and above the dark blue curve. We see that excessively low values of $m$ restrict severely the possible values of 
$a,\alpha$. The points should be under the red curve if possible, since it would imply that $\mu_1$ 
is increasing, which assures that $\p \left ( \abs{\eta(X,X')-Q^*_\alpha } \le t \right )$ is smooth on
a larger neighborhood of $0$. Estimators of this quantity are displayed in \cref{fig:4}
for $\alpha = 0.26$, $m=0.35$.
\begin{figure}
\centering
\includegraphics[width=0.6\textwidth]{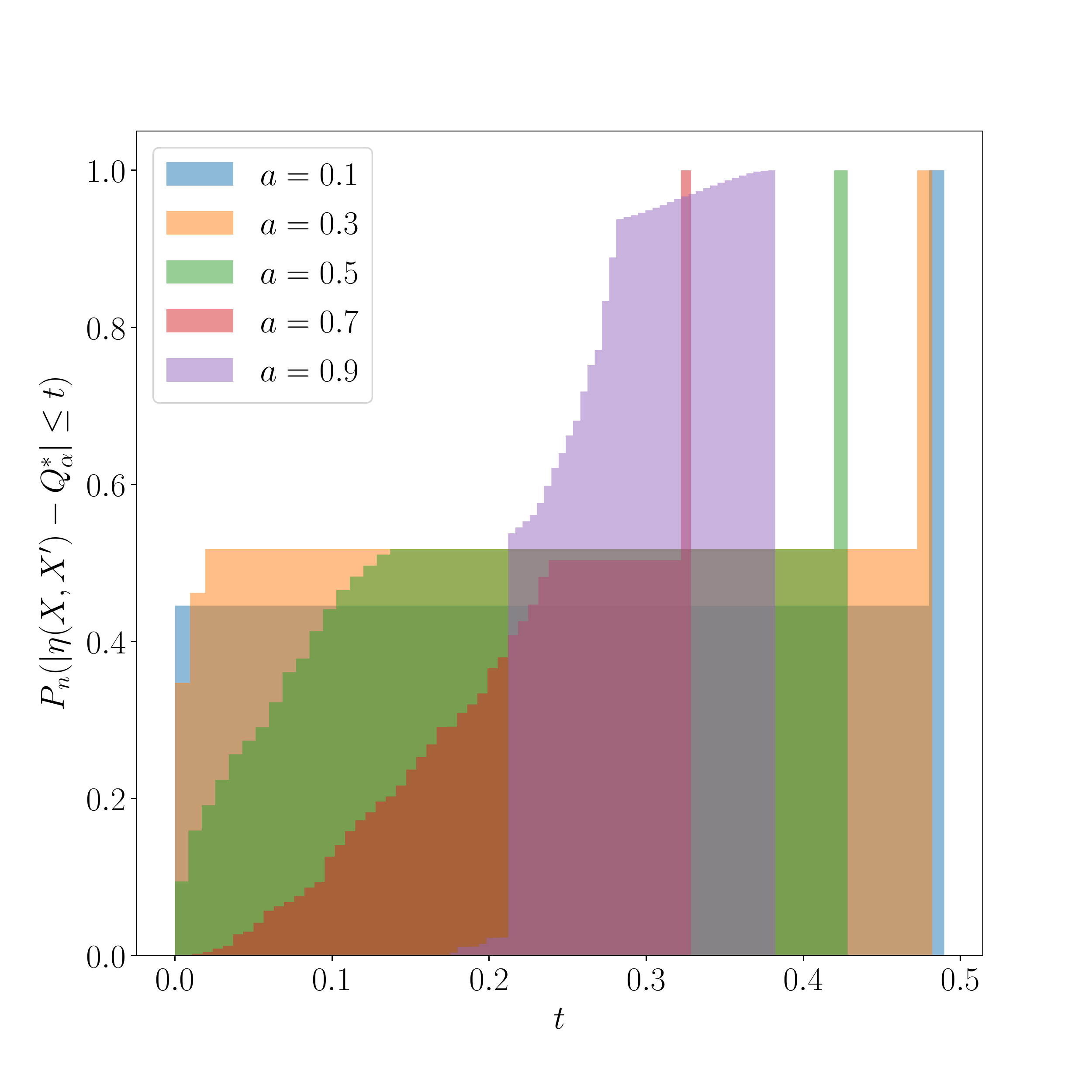}
\caption{Empirical Mammen-Tsybakov distributions for different values of $a$ when $\alpha = 0.26$, $m=0.35$.}
\label{fig:4}
\end{figure}

Now that the distribution of the data is set, we need to set the class of functions on which we optimize 
\cref{simlearn}.
For all $t \in [0,1]$, we define the proposal family $\mathcal{S}_0$ as $\left \{  x,x' \mapsto \I \{x,x' \in S_t \} 
\, | \, 0 \le t \le 1 \right \}$ where 
\begin{align*}
S_t = \left \{  x,x' \in \X\times\X \; \mid \; \left ( x \le t \cap x' \le t \right ) 
 \cup  \left ( 1-x \le t \cap 1-x' \le t \right )  \right \}.
\end{align*}
The sets $S_t$ are illustrated by \cref{fig:5}.

The risks $R^+(S)$, $R^-(S)$ of an element $S$ of $\mathcal{S}$ can be expressed in closed form with the expression 
of $\eta$. Indeed,
\begin{align*}
R^+(S) &=  2  \int_{ S_t  }\eta(x,x') dx dx', \\
&= \lambda(S_t) + \int_{ S_t  } \left ( \mu_1(x) - 1\right ) \left ( \mu_1(x')-1 \right) dx dx'.
\end{align*}
using \cref{two-classes-eta}. The right-hand side integral is easily developped since it is a sum of integrals 
over squares included in $[0,1] \times [0,1]$.

We now describe the processus of choosing an optimal empirical function $\hat{S}_n$  for a set of 
observations $(X_i,Y_i)$. For all pairs $X_i, X_j$, we derive the quantity $S_{i,j} = 
 \min \left ( \max \left (1-X_i,1-X_j \right ) , \max \left ( X_i,X_j \right )  \right)$.
 \begin{figure}
\centering
\includegraphics[width=0.6\textwidth]{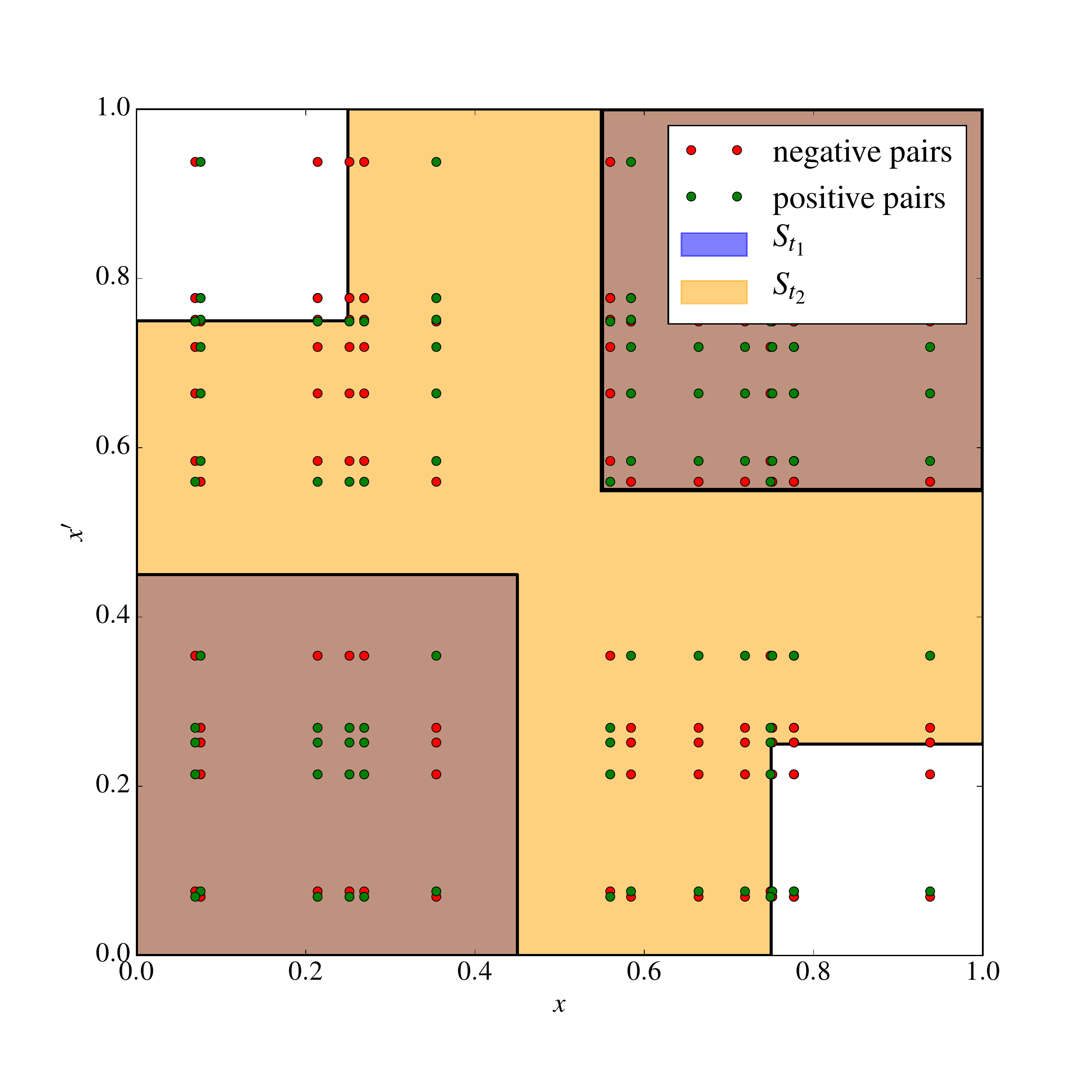}
\caption{Representation of proposal regions $S_{t_1}$, $S_{t_2}$ for $0 < t_1 < \frac{1}{2} < t_2 < 1$. 
Note that $S_{t_1} \subset S_{t_2}$.}
\label{fig:5}
\end{figure}
 Let $\sigma \left\{ 1, \dots , n(n-1)/2 \right\} \mapsto \{1,\dots , n\}^2 $ be the function that
 orders the quantities $S_{i,j}$ increasingly, i.e. $S_{\sigma(1)} \le \dots \le S_{\sigma(n(n-1)/2)}$.
Choosing an optimal empirical function $\hat{S}_n$ in $\mathcal{S}$ requires to solve the 
pointwise ROC optimization problem for $\left ( S_{i,j}, Z_{i,j} \right )_{i<j}$ 
and proposal functions 
\begin{align*}
\left \{ x \mapsto \I\left\{ x \le \frac{S_{\sigma(i)} + S_{\sigma(i+1)}}{2}\right\} 
\; \Big | \; 0\le i \le \frac{n(n-1)}{2}
\right \},
\end{align*}
where $S_{\sigma(0)} = 0$ and $S_{\sigma((n(n-1)/2) + 1)} = 1$ by convention. 
It can be solved in $O \left (n^2 \log n \right ) $ time.
 \begin{figure}
\centering
\includegraphics[width=0.7\textwidth]{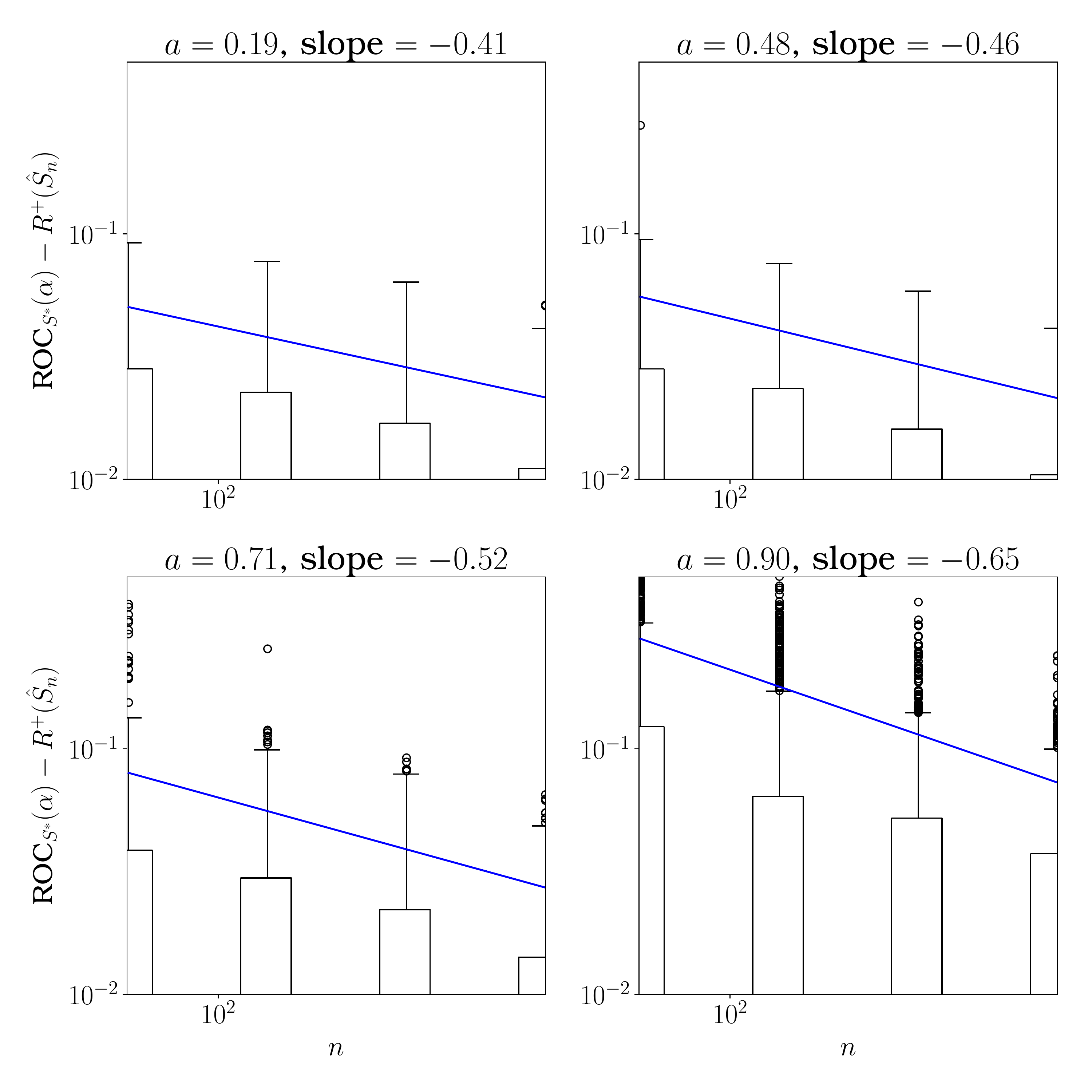}
\caption{Boxplot of the 1000 regrets $\roc_{S^*}(\alpha) - R^+(\hat{S}_n)$ for each $n$ and several values of $a$. 
The line represents the regression on the 90-quantile.}
\label{fig:6}
\end{figure}

For all $a\in \{1/10, \dots, 9/10\}$, we generate $512$ data points and compute the generalization error 
$\roc_{S^*} (\alpha) - R^+ (\hat{S}_n)$ for the $n$ first data points, where $n \in \{64, 128, 256, 512\}$, 
and repeat the operation $1000$ times. We introduce $Q_{a,n}$ as the 90-quantile of the 1000 realisations of 
$ \roc_{S^*}( \alpha ) -  R^+(\hat{S}_n) $ for a given $(n, a)$. The coefficients $(C_{a},D_{a})$ of the regression 
$ Q_{a,n} = D_{a} + C_{a} \times \log (n)$ are estimated. \cref{fig:fast_gen_speeds} shows the $C_{a}$'s 
given the $a$'s. The estimation of the $C_a$'s is illustrated by \cref{fig:6}.



\end{document}